\documentclass[12pt]{article}
\usepackage{amsmath}
\usepackage{graphicx,psfrag,epsf,epsfig}
\usepackage{enumerate}
\usepackage{natbib}
\usepackage{algorithm}
\usepackage{algorithmic}
\usepackage{amssymb}
\usepackage{graphicx}
\usepackage{algorithm}
\usepackage{algorithmic}
\usepackage{amsthm}
\usepackage{amsmath}
\usepackage{enumerate}
\usepackage{thmtools,thm-restate}
\newcommand{\blind}{0}
\pdfoutput=1
\begin{document}

\bibliographystyle{natbib}

\def\spacingset#1{\renewcommand{\baselinestretch}{#1}\small\normalsize} \spacingset{1}


\if0\blind
{
  \date{}
  \title{\bf Graph Structure Learning from Unlabeled Data for Event Detection}
  \author{Sriram Somanchi\thanks{somanchi.1@nd.edu}
  \hspace{.2cm} \\
  Mendoza College of Business, University of Notre Dame
  \and 
  Daniel B. Neill\thanks{neill@cs.cmu.edu} 
  \hspace{.2cm}\\
  Event and Pattern Detection Laboratory, Carnegie Mellon University}
  \maketitle
} \fi

\if1\blind
{
  \bigskip
  \bigskip
  \bigskip
  \begin{center}
    {\LARGE\bf Graph Structure Learning from Unlabeled Data for Event Detection}
\end{center}
  \medskip
} \fi

\bigskip
\begin{abstract}
Processes such as disease propagation and information diffusion often spread over some latent network structure which must be learned from observation.
Given a set of unlabeled training examples representing occurrences of an event type of interest (e.g., a disease outbreak),
our goal is to learn a graph structure that can be used to accurately detect future events of that type.  Motivated by new theoretical results on the consistency of constrained and unconstrained subset scans, we propose a novel framework for learning graph structure from unlabeled data by comparing the most anomalous subsets detected with and without the graph constraints. Our framework uses the mean normalized log-likelihood ratio score to measure the quality of a graph structure, and efficiently searches for the highest-scoring graph structure. Using simulated disease outbreaks injected into real-world Emergency Department data from Allegheny County, we show that our method learns a structure similar to the true underlying graph, but enables faster and more accurate detection.
\end{abstract}

\noindent%

\spacingset{1.45}


\newcommand{\ignore}[1]{}
\section{Introduction}
\label{}
Event detection in massive data sets has applications to multiple
domains, such as information diffusion or detecting disease outbreaks.
In many of these domains, the data has an underlying graph or network
structure: for example, an outbreak might spread via person-to-person
contact, or the latest trends might propagate through a social network.
In the typical, graph-based event detection problem, we
are given a graph structure $G=(V,E)$ and a time series of observed
counts for each graph node $v_i$, and must detect connected
subgraphs where the recently observed counts are significantly higher
than expected.  For example, public health officials wish to achieve
early and accurate detection of emerging outbreaks by identifying
connected regions (e.g., subsets of spatially adjacent zip codes $v_i$)
with anomalously high counts of disease cases.

Assuming that the graph structure is known, various graph-based event detection methods~\citep{patil04} can be used to detect anomalous
subgraphs.  We review these methods in \S\ref{gbed} below.  Typically, however, the network structure is
\emph{unknown}.  For example, the spread of disease may be influenced not only by spatial adjacency but also by commuting patterns (e.g., individuals
work downtown but live in a suburb), contamination of food or water sources, animal migrations, or other factors.  Assuming an incorrect graph
structure can result in less timely and less accurate event detection, since the affected areas may be disconnected and hence may not be identified
as an anomalous subgraph.  In such cases, \emph{learning} the correct graph structure (e.g., from historical data) has the potential to dramatically
improve detection performance.

Thus we consider the graph-based event detection problem in the case
where the true graph structure $G_T$ is unknown and must be inferred from data.
To learn the graph, we are given a set of training examples $\{D_1 \ldots
D_J\}$, where each example $D_j$ represents a different ``snapshot'' of
the data when an event is assumed to be occurring in some subset of
nodes that is connected given the (unknown) graph structure.  We assume
that training examples are generated from some underlying distribution
on the true latent graph structure, and wish to accurately detect future events
drawn from that same distribution.  Thus our goal is to learn a graph
structure that minimizes detection time and maximizes accuracy when
used as an input for event detection.

Several recent methods~\citep{Gomez-Redriguez:2010:NetInf,Myers:2010:LSNI, Gomez-Redriguez:2012:MultiTree} learn an underlying graph structure using
\emph{labeled} training data, given the true affected subset of nodes $S_j^T$ for each training example $D_j$. However, in many cases labeled data is
unavailable: for example, public health officials might be aware that an outbreak has occurred, but may not know which areas were affected and when.
Hence we focus on learning graph structure from \emph{unlabeled} data, where the affected subset of nodes $S_j^T$ for each training example is not
given, and we observe only the observed and expected counts at each node. In the remainder of this paper, we present a novel framework for graph
structure learning from unlabeled data, and show that the graphs learned by our approach enable more timely and more accurate event detection.  We support these
empirical evaluations with new theoretical results on the consistency of constrained and unconstrained subset scans, as described in \S3 and \S4.4 below.

\subsection{Graph-Based Event Detection}\label{gbed}

Given a graph $G=(V,E)$ and the observed and expected counts at each graph node, existing methods for \emph{graph-based event detection} can be used to identify the most anomalous
connected subgraph.  Here we focus on the \emph{spatial scan} framework for event detection, which was first developed by~\citet{kulldorff97a}, building on work
by~\citet{naus65} and others, and extended to graph data by~\citet{patil04}.  These methods maximize the \emph{log-likelihood ratio statistic} $F(S) = \log
\frac{\mbox{Pr}(\mbox{Data} \:|\: H_1(S))}{\mbox{Pr}(\mbox{Data} \:|\: H_0)}$ over connected subgraphs $S$.  Searching over connected subgraphs, rather than clusters of fixed shape
such as circles~\citep{kulldorff97a} or rectangles~\citep{neill04b}, can increase detection power and accuracy for irregularly shaped spatial clusters.

In this paper, we assume that the score function $F(S)$ is an \emph{expectation-based scan statistic}~\citep{neill05b}.  The null hypothesis $H_0$ assumes that no events are
occurring, and thus each observed count $x_i$ is assumed to be drawn from some distribution with mean equal to the expected count $\mu_i$: $x_i \sim \mbox{Dist}(\mu_i)$.  The
alternative hypothesis $H_1(S)$ assumes that counts in subgraph $S$ are increased by some constant multiplicative factor $q>1$: $x_i \sim \mbox{Dist}(q\mu_i)$ for $v_i \in S$, and $x_i \sim
\mbox{Dist}(\mu_i)$ for $v_i \not\in S$, where $q$ is chosen by maximum likelihood estimation.  We further assume that $\mbox{Dist}$ is some distribution in the \emph{separable
exponential family}~\citep{neill-ltss}, such as the Poisson, Gaussian, or exponential.  This assumption enables efficient identification of the highest-scoring connected subgraph
and highest-scoring unconstrained subset, which will be important components of our graph structure learning framework described below.  Our evaluation
results below assume the expectation-based Poisson statistic~\citep{neill05b}.  In this case, the log-likelihood ratio score can be computed as $F(S) = C \log (C/B) + B - C$, if
$C>B$, and 0 otherwise, where $C = \sum_{v_i \in S} x_i$ and $B = \sum_{v_i\in S} \mu_i$.

Maximizing the log-likelihood ratio statistic $F(S)$ over connected subgraphs is a challenging computational problem for which multiple algorithmic approaches exist.  The two main
methods we consider in this paper are GraphScan~\citep{speakman14} and Upper Level Sets (ULS)~\citep{patil04}.  GraphScan is guaranteed to find the highest-scoring connected subgraph
for the expectation-based scan statistics considered here, but can take exponential time in the worst case.  ULS scales quadratically with graph size, but is a heuristic that is not
guaranteed to find the optimal subgraph. GraphScan requires less than a minute of computation time for a $\sim$100 node graph, and improves detection power as compared to ULS, but is
computationally infeasible for graphs larger than $200$ to $300$ nodes~\citep{speakman14}.  We also note that the previously proposed FlexScan method \citep{tango05} identifies
subgraphs nearly identical to those detected by GraphScan, but is computationally infeasible for graphs larger than $\sim$30 nodes.

As shown by~\citet{speakman14}, the detection performance of GraphScan and ULS is often improved by incorporating proximity as well as connectivity constraints, thus
preventing these methods from identifying highly irregular tree-like structures.  To do so, rather than performing a single search over the entire graph, we perform separate searches
over the ``local neighborhood'' of each of the $N$ graph nodes, consisting of that node and its $k-1$ nearest neighbors for some constant $k$.  We then report the highest-scoring connected subgraph over
all local neighborhoods.

\section{Problem Formulation}

Our framework for graph learning takes as input a set of training examples $\{D_1\ldots D_J\}$, assumed to be independently drawn from some
distribution $\mathbf{D}$.  For each example $D_j$, we are given the observed count $x_i$ and expected count $\mu_i$ for each graph node $v_i$, $i = 1 \ldots
N$.  We assume that each training example $D_j$ has an set of affected nodes $S_j^T$ that is a connected subgraph of the true underlying graph
structure $G_T$; note that both the true graph $G_T$ and the subgraphs $S_j^T$ are unobserved.  Unaffected nodes $v_i \not\in S_j^T$ are assumed to have
counts $x_i$ that are drawn from some distribution with mean $\mu_i$, while affected nodes $v_i \in S_j^T$ are assumed to have higher counts.  Given
these training examples, we have three main goals:\\[1ex]
1) Accurately estimate the true underlying graph structure $G_T$.  Accuracy of graph learning is measured by the precision and recall of the learned set of graph edges $G^\ast$ as compared to the true graph $G_T$.\\[1ex]
2) Given a separate set of test examples $\{D_1\ldots D_J\}$ drawn from $\mathbf{D}$, identify the affected subgraphs $S_j^T$.  Accuracy of detection is measured by the average overlap coefficient between the true and identified subgraphs.\\[1ex]
3) Distinguish test examples drawn from $\mathbf{D}$ from examples with no affected subgraph ($S_j^T = \emptyset$).  Detection power is measured by the true positive rate (proportion of correctly identified test examples) for a fixed false positive rate (proportion of incorrectly identified null examples). \\ \\
The second and third performance measures assume that the learned graph $G^\ast$ is used as an input for a graph-based event detection method such as GraphScan, and that
method is used to identify the highest scoring connected subgraph of $G^\ast$ for each test example.

A key insight of our graph learning framework is to evaluate the quality
of each graph structure $G_m$ ($m$ denotes number of edges in the graph)
by comparing the most anomalous subsets
detected with and without the graph constraints.  For a given training example $D_j$, we
can use the fast subset scan~\citep{neill-ltss} to identify the
highest-scoring unconstrained subset $S_j^\ast = \arg\max_{S\subseteq V} F(S)$, with score
$F_j = F(S_j^\ast)$.  This can be done very efficiently, evaluating a number of subsets that is linear rather than exponential
in the number of graph nodes, for any function satisfying the linear-time subset scanning property~\citep{neill-ltss},
including the expectation-based scan statistics considered here.  We can use either GraphScan~\citep{speakman14} or ULS~\citep{patil04} to estimate the highest-scoring connected subgraph $S_{mj}^\ast = \arg\max_{S \subseteq V:\:S
\:\mbox{\small{connected in}}\: G_m} F(S)$, with score $F_{mj} =
F(S_{mj}^\ast)$.  We then compute the \emph{mean normalized score}
$\bar F_{norm}(G_m) = \frac{1}{J} \sum_{j=1\ldots J}
\frac{F_{mj}}{F_j}$, averaged over all $J$ training examples, as a
measure of graph quality.

As noted above, we assume that the affected subset of nodes for each
training example is a connected subgraph of the true (unknown) graph
structure $G_T$.  Intuitively, if a given graph $G_m$ is similar to
$G_T$, then the maximum connected subgraph score $F_{mj}$ will be close
to the maximum unconstrained subset score $F_j$ for many training
examples, and $\bar F_{norm}(G_m)$ will be close to 1.  On the
other hand, if graph $G_m$ is missing essential connections, then we
expect the values of $F_{mj}$ to be much lower than the corresponding
$F_j$, and $\bar F_{norm}(G_m)$  will be much lower than 1.
Additionally, we would expect a graph $G_m$ with high scores $F_{mj}$ on
the training examples to have high power to detect future events drawn
from the same underlying distribution.  However, any graph with a large
number of edges will also score close to the maximum unconstrained
score.  For example, if graph $G_m$ is the complete graph on $N$ nodes,
all subsets are connected, and $F_{mj} = F_j$ for all training examples
$D_j$, giving $\bar F_{norm}(G_m) = 1$.  Such
under-constrained graphs will produce high scores $F_{mj}$ even when
data is generated under the null hypothesis, resulting in reduced
detection power. Thus we wish to optimize the tradeoff between higher
mean normalized score and lower number of edges $m$.  Our solution is to
compare the mean normalized score of each graph structure $G_m$ to the
distribution of mean normalized scores for random graphs with the same
number of edges $m$, and choose the graph with the most significant
score given this distribution.

\section{Theoretical Development}

In this section, we provide a theoretical justification for using the mean
normalized score, $\bar F_{norm}(G_m) = \frac{1}{J} \sum_{j=1\ldots J}
\frac{F_{mj}}{F_j}$, as a measure of the quality of graph $G_m$.  Our key result is a proof that the expected value $E\left[\frac{F_{mj}}{F_j}\right] = 1$ if and only if graph $G_m$ contains the
true graph $G_T$, assuming a sufficiently strong and homogeneous signal.  More precisely, let us assume the following: \\[1ex]
\emph{(A1)} Each training example $D_j$ has an affected subset $S_j^T$ that is a connected subgraph of $G_T$.  Each $D_j$ is an independent random draw from some distribution $\mathbf{D}$, where
each connected subgraph $S_j^T$ is assumed to have some non-zero probability $P_j$ of being affected.\\
\emph{(A2)} The score function $F(S)$ is an expectation-based scan statistic in the separable exponential family.  Many distributions, such as the Poisson, Gaussian, and exponential, satisfy this property.\\[1ex]
Now, for a given training example $D_j$, we define the \emph{observed excess risk} $g_{ij} = \frac{x_i}{\mu_i}-1$ for each node $v_i$.
Let $r_{\max}^{\text{aff},j} = \max_{v_i \in S_j^T} g_{ij}$
and $r_{\min}^{\text{aff},j} = \min_{v_i \in S_j^T} g_{ij}$ denote the maximum and minimum of the observed excess risk over affected nodes, and $r_{\max}^{\text{unaff},j} = \max_{v_i \not\in S_j^T} g_{ij}$ denote the
maximum of the observed excess risk over unaffected nodes, respectively.  We say that the signal for training example $D_j$ is $\alpha$-strong if and only if $r_{\min}^{\text{aff},j} > \alpha r_{\max}^{\text{unaff},j}$,
and we say that the signal for training example $D_j$ is $\alpha$-homogeneous if and only if $r_{\max}^{\text{aff},j} < \alpha r_{\min}^{\text{aff},j}$.
We also define the \emph{signal size} for training example $D_j$, $\eta_j = \frac{\sum_{v_i \in S_j^T} \mu_i}{\sum_{v_i} \mu_i} \le 1$.  Given assumptions (A1)-(A2) above, we can show:
\newtheorem{lemma}{Lemma}
\begin{restatable}{lemma}{twolemma}
\label{lemma_superset}
For each training example $D_j$, there exists a constant $\alpha_j > 1$
such that, if the signal is $\alpha_j$-homogeneous and 1-strong, then
the highest scoring unconstrained subset $S_j^{\ast} \supseteq S_j^T$.
We note that $\alpha_j$ is a function of $r_{\max}^{\text{aff},j}$, and $\alpha_j \ge 2$ for the Poisson, Gaussian, and exponential distributions.
\end{restatable}
\begin{restatable}{lemma}{alphalemma}
\label{lemma_subset}
For each training example $D_j$, there exists a constant $\beta_j > 1$
such that, if the signal is $\frac{\beta_j}{\eta_j}$-strong, then the
highest scoring unconstrained subset $S_j^{\ast} \subseteq S_j^T$.
We note that $\beta_j$ is a function of $r_{\max}^{\text{unaff},j}$, and $\beta_j \le 2$ for the Gaussian distribution.
\end{restatable}
\noindent Proofs of Lemma 1 and Lemma 2 are provided in the Appendix.
\newtheorem{theorem}{Theorem}
\begin{theorem}
\label{first_theorem}
If the signal is $\alpha_j$-homogeneous and $\frac{\beta_j}{\eta_j}$-strong for all training examples
$D_j \sim \mathbf{D}$, then the following properties hold for the assumed graph
$G_m$ and true graph $G_T$:

a) If $G_T \setminus G_m = \emptyset$ then $E\left[\frac{F_{mj}}{F_j}\right] = 1$.

b) If $G_T \setminus G_m \ne \emptyset$ then $E\left[\frac{F_{mj}}{F_j}\right] < 1$.
\end{theorem}
\begin{proof}
Lemmas \ref{lemma_superset} and \ref{lemma_subset} imply that $S_j^{\ast} = S_j^T$ for all $D_j \sim \mathbf{D}$.  For part a), $G_T \setminus G_m = \emptyset$ implies that the affected subgraph $S_j^T$ (which is assumed to be connected in $G_T$) is connected in $G_m$ as well.  Thus $S_{mj}^{\ast} = S_j^T$, and $\frac{F_{mj}}{F_j} = 1$ for all $D_j \sim \mathbf{D}$.  For part b), $G_T \setminus G_m \ne \emptyset$ implies that there exists some pair of nodes $(v_1,v_2)$ such that $v_1$ and $v_2$ are connected in $G_T$ but not in $G_m$. By assumption (A1), the subset $S_j^T = \{v_1,v_2\}$ has non-zero probability $P_j$ of being generated, and we know $S_j^{\ast} = \{v_1,v_2\}$, but $S_{mj}^{\ast} \ne \{v_1,v_2\}$ since the subset is not connected in $G_m$.  Since the signal is $\alpha_j$-homogeneous and $\frac{\beta_j}{\eta_j}$-strong, we observe that $S_j^{\ast}$ is the unique optimum.  Thus we have $F_{mj} < F_j$ for that training example, and $ E\left[\frac{F_{mj}}{F_j}\right] \le 1-P_j\left(1-\frac{F_{mj}}{F_j}\right) < 1$.
\end{proof}

\section{Learning Graph Structure}

We can now consider the mean normalized score $\bar F_{norm}(G_m) = \frac{1}{J} \sum_{j=1\ldots J}
\frac{F_{mj}}{F_j}$ as a measure of graph quality, and for each number of edges $m$, we can search for the graph
$G_m$ with highest mean normalized score. However, it is computationally infeasible to search exhaustively over
all $2^{\frac{|V|(|V|-1)}{2}}$ graphs.  Even computing the mean normalized score of a single graph $G_m$ may
require a substantial amount of computation time, since it requires calling a graph-based event detection method
such as Upper Level Sets (ULS) or GraphScan to find the highest-scoring connected subgraph for each training
example $D_j$.  In our general framework for graph structure learning, we refer to this call as
BestSubgraph($G_m$, $D_j$), for a given graph structure $G_m$ and training example $D_j$.  Either ULS or GraphScan
can be used to implement BestSubgraph, where ULS is faster but approximate, and GraphScan is slower but guaranteed
to find the highest-scoring connected subgraph.  In either case, to make graph learning computationally tractable,
we must \emph{minimize} the number of calls to BestSubgraph, both by limiting the number of graph structures under
consideration, and by reducing the average number of calls needed to evaluate a given graph.

Thus we propose a \emph{greedy} framework for efficient graph structure learning that starts with the complete
graph on $N$ nodes and sequentially removes edges until no edges remain (Algorithm 1).  This procedure produces a
sequence of graphs $G_m$, for each $m$ from $M=\frac{N(N-1)}{2}$ down to 0. For each graph $G_m$, we produce graph
$G_{m-1}$ by considering all $m$ possible edge removals and choosing the one that maximizes the mean normalized
score.  We refer to this as BestEdge($G_m$, $D$), and consider three possible implementations of BestEdge in \S4.1
below.  Once we have obtained the sequence of graphs $G_0 \ldots G_M$, we can then use randomization testing to
choose the most significant graph $G_m$, as described in \S4.2.  The idea of this approach is to remove
unnecessary edges, while preserving essential connections which keep the maximum connected subgraph score close to
the maximum unconstrained subset score for many training examples.

However, a naive implementation of greedy search would require $O(N^4)$ calls to BestSubgraph, since $O(N^2)$
graph structures $G_{m-1}$ would be evaluated for each graph $G_m$ to choose the next edge for removal.  Even a
sequence of random edge removals would require $O(N^2)$ calls to BestSubgraph, to evaluate each graph $G_0 \ldots
G_M$.  Our efficient graph learning framework improves on both of these bounds, performing exact or approximate
greedy search with $O(N^3)$ or $O(N\log N)$ calls to BestSubgraph respectively.  The key insight is that removal
of an edge only requires us to call BestSubgraph for those examples $D_j$ where removing that edge disconnects the
highest scoring connected subgraph.  See \S4.3 for further analysis and discussion.

\vspace{-0.4cm}
\begin{center}
\begin{algorithm}[t]
\caption{Graph structure learning framework}
\label{alg}
\begin{algorithmic}[1]
\STATE Compute correlation $\rho_{ik}$ between each pair of nodes
$v_i$ and $v_k$, $i\neq k$.  These will be used in step 5.
\STATE Compute highest-scoring unconstrained subset $S_j^\ast$ and its
score $F_j$ for each example $D_j$ using the fast subset scan~\citep{neill-ltss}.
\STATE For $m=\frac{N(N-1)}{2}$, let $G_m$ be the complete graph on $N$
nodes.  Set $S_{mj}^\ast = S_j^\ast$ and $F_{mj} = F_j$ for all training examples
$D_j$, and set $\bar F_{norm}(G_m) = 1$.
\WHILE{number of remaining edges $m > 0$}
\STATE Choose edge $e_{ik} = \mathrm{BestEdge}(G_m,D)$, and set $G_{m-1} =
G_m$ with $e_{ik}$ removed.
\FOR{each training example $D_j$}
\STATE  If removing edge $e_{ik}$ disconnects subgraph $S_{mj}^\ast$, then set
$S_{m-1,j}^\ast = \mathrm{BestSubgraph}(G_{m-1},D_j)$ and $F_{m-1,j} =
F(S_{m-1,j}^\ast)$.  Otherwise set $S_{m-1,j}^\ast = S_{mj}^\ast$ and $F_{m-1,j} =
F_{mj}$.
\ENDFOR
\STATE Compute $\bar F_{norm}(G_{m-1}) = \frac{1}{J} \sum_{j = 1\ldots
J} \frac{F_{m-1,j}}{F_j}$.
\STATE $m \leftarrow m-1$
\ENDWHILE
\STATE Repeat steps 3-11 for $R$ randomly generated sequences of edge
removals to find the most significant graph $G_m$.
\end{algorithmic}
\end{algorithm}
\end{center}
\vspace{-0.4cm}

\subsection{Edge Selection Methods}

Given a graph $G_m$ with $m$ edges, we consider three methods $\mathrm{BestEdge}(G_m,D)$ for choosing the next
edge $e_{ik}$ to remove, resulting in the next graph $G_{m-1}$.  First, we consider an exact greedy search.  We
compute the mean normalized score $\bar F_{norm}(G_{m-1})$ resulting from each possible edge removal $e_{ik}$, and
choose the edge which maximizes $\bar F_{norm}(G_{m-1})$. As noted above, computation of the mean normalized score
for each edge removal is made efficient by evaluating the score $F_{m-1,j}$ only for training examples $D_j$ where
removing edge $e_{ik}$ disconnects the highest scoring subgraph. The resulting graph $G_{m-1}$ will have $\bar
F_{norm}(G_{m-1})$ as close as possible to $\bar F_{norm}(G_m)$.  We show in \S\ref{comp_sec} that only
$O(N)$ of the $O(N^2)$ candidate edge removals will disconnect the highest scoring subgraphs, reducing the number
of calls to BestSubgraph from quartic to cubic in $N$. However, this still may result in overly long run times,
necessitating the development of the alternative approaches below.

In the early stages of the greedy edge removal process, when the number
of remaining edges $m$ is large, many different edge removals $e_{ik}$
might not disconnect any of the subgraphs $S_{mj}^\ast$, and all such graphs
would have the same mean normalized score $\bar F_{norm}(G_{m-1}) = \bar
F_{norm}(G_m)$.  To avoid removing potentially important edges, we must
carefully consider how to break ties in mean normalized score.  In this
case, we choose the edge $e_{ik}$ with lowest \emph{correlation} between
the counts at nodes $v_i$ and $v_k$.  If two nodes are connected to each
other in the latent graph structure over which an event spreads, we
expect both nodes to often be either simultaneously affected by an event
in that part of the network, or simultaneously unaffected by an event in
some other part of the network, and hence we expect the observed counts
in these nodes to be correlated. Hence, if the Pearson correlation
$\rho_{ik}$ between two nodes $v_i$ and $v_k$ is very low, the
probability that the two nodes are connected is small, and thus edge
$e_{ik}$ can be removed.  We refer to the resulting algorithm, removing
the edge $e_{ik}$ which reduces the mean normalized score the least, and
using correlation to break ties, as the Greedy Correlation (GrCorr)
method.

Our second approach is based on the observation that GrCorr would require
$O(m)$ calls to BestSubgraph for each graph $G_m$, $m=1\ldots M$, which
may be computationally infeasible depending on the graph size and the
implementation of BestSubgraph.  Instead, we use the fact that $F_{m-1,j}
= F_{mj}$ if removing edge $e_{ik}$ does not disconnect subgraph $S_{mj}^\ast$,
and $F_{m-1,j} < F_{mj}$ otherwise.  To do so, we \emph{count} the number
of subgraphs $S_{mj}^\ast$, for $j=1\ldots J$, which would be disconnected by
removing each possible edge $e_{ik}$ from graph $G_m$, and choose the
$e_{ik}$ which disconnects the \emph{fewest} subgraphs. The resulting
graph $G_{m-1}$ is expected to have a mean normalized score $\bar
F_{norm}(G_{m-1})$ which is close to $\bar F_{norm}(G_m)$, since
$F_{m-1,j} = F_{mj}$ for many subgraphs, but this approach does not
guarantee that the graph $G_{m-1}$ with highest mean normalized score will
be found.  However, because we choose the edge $e_{ik}$ for which the
fewest subgraphs $S_{mj}^\ast$ are disconnected, and only need to call
BestSubgraph for those examples $D_j$ where removing $e_{ik}$ disconnects
$S_{mj}^\ast$, we are choosing the edge $e_{ik}$ which requires the
\emph{fewest} calls to BestSubgraph for each graph $G_m$.  Again,
correlation is used to break ties: if two edge removals $e_{ik}$
disconnect the same number of subgraphs, the edge with lower correlation
is removed.  We refer to this as Pseudo-Greedy Correlation (PsCorr), and we
show in \S\ref{comp_sec} that this approach reduces the number of calls
to BestSubgraph from $O(N^3)$ to $O(N\log N)$ as compared to exact greedy search.

In our empirical results below, we compare GrCorr and PsCorr to a simple implementation of
$\mathrm{BestEdge}(G_m,D)$, which we refer to as Correlation (Corr).
Corr chooses the next edge removal $e_{ik}$ to be the edge with the
lowest value of $\rho_{ik}$, and hence the greedy edge removal approach
corresponds to keeping all edges with correlation above some threshold
$\rho$.  Our empirical results, presented below, demonstrate that GrCorr
and PsCorr significantly improve timeliness and accuracy of event
detection as compared to Corr.

\subsection{Finding the Most Significant Graph}

\begin{figure}
\label{comp_score_dist}
\begin{center}
\includegraphics[scale=0.45]{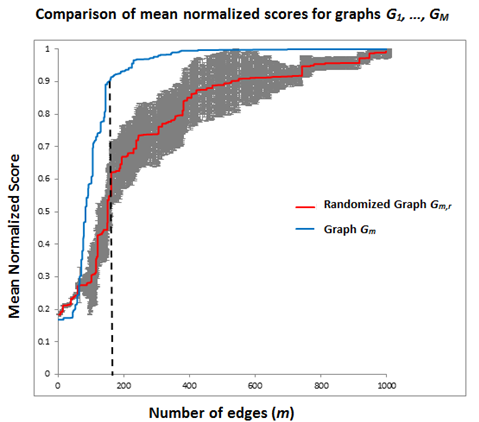}
\caption{Example of finding the most significant graph.  Blue
line: mean normalized score $\bar F_{norm}(G_m)$ for each graph $G_1
\ldots G_M$.  Red line and grey shadow: mean and standard deviation of
$\bar F_{norm}(G_{m,r})$ for randomized graphs with $m$ edges.  Dashed
line: most significant graph $G_m^\ast$.}
\end{center}
\vspace{-0.4cm}
\end{figure}

Our proposed graph structure learning approach considers a set of nested
graphs $\{G_1 \ldots G_M\}$, $M = \frac{N(N-1)}{2}$, where graph $G_m$ has $m$
edges and is formed by removing an edge from graph $G_{m+1}$.
We note that, for this set of graphs, $\bar F_{norm}(G_m)$ is
monotonically increasing with $m$, since the highest scoring connected
subgraph $S_{mj}^\ast$ for graph $G_m$ will also be connected for graph
$G_{m+1}$, and thus $F_{m+1,j} \ge F_{mj}$ for each training example
$D_j$.  Our goal is to identify the graph $G_m$ with the best tradeoff between a
high mean normalized score $\bar F_{norm}(G_m)$ and a small number of
edges $m$, as shown in Figure 1.  Our solution is to generate a large
number $R$ of \emph{random} permutations of the $M = \frac{N(N-1)}{2}$
edges of the complete graph on $N$ nodes.  For each permutation $r=1\ldots
R$, we form the sequence of graphs $G_{1,r} \ldots G_{M,r}$ by removing
edges in the given random order, and compute the mean normalized score of
each graph. For a given number of edges $m$, we compute the mean $\mu_m$
and standard deviation $\sigma_m$ of the mean normalized scores of the $R$
random graphs with $m$ edges.  Finally we choose the graph $G_m^\ast =
\arg\max_m \frac{\bar F_{norm}(G_m) - \mu_m}{\sigma_m}$.  This ``most
significant graph'' has the most anomalously high value of $\bar
F_{norm}(G_m)$ given its number of edges $m$.  Ideally, in order to compute
the most significant graph structure, we want to compare our mean normalized
score to the mean normalized score of any random graph with the same number of edges.  However, due to
the computational infeasibility of scoring all the random graph structures with varying number
of edges, we instead choose random permutations of edges to be removed.
\ignore{
Alternatively, we could potentially split the training dataset into training and
validation data.  We could use the training data to decide on which edge to remove
and we could use the validation data to generate the mean normalized score at
each edge removal.  We noticed in our experiments that this split tends to
reduce the effective number of training examples, thereby affecting the performance of
our learned graph structures (refer to \S\ref{num_train_performance} for more
details).  Further, we also performed additional experiments where we used twice as much
data (same amount of training data, plus additional examples for validation) and
there was not a huge improvement in performance of
our learned graph structures.  Hence, we used the same data to decide the order of
edge removals and to generate the mean normalized score.}

\subsection{Computational Complexity Analysis}
\label{comp_sec}

We now consider the computational complexity of each step of our graph
structure learning framework (Alg.~1), in terms of the number of
nodes $N$, number of training examples $J$, and number of randomly
generated sequences $R$.  Step 1 (computing correlations) requires $O(J)$
time for each of the $O(N^2)$ pairs of nodes.  Step 2 (computing the
highest-scoring unconstrained subsets) requires $O(N \log N)$ time for
each of the $J$ training examples, using the linear-time subset scanning
method~\citep{neill-ltss} for efficient computation.  Steps 5-10 are
repeated $O(N^2)$ times for the original sequence of edges and $O(N^2)$
times for each of the $R$ randomly generated sequences of edges.  Within
the loop, the computation time is dominated by steps 5 and 7, and depends
on our choice of $\mathrm{BestSubgraph}(G,D)$ and $\mathrm{BestEdge}(G,D)$.

For each call to BestSubgraph, GraphScan requires worst-case exponential
time, approximately $O(1.2^N)$ based on empirical results by~\citet{speakman14},
while the faster, heuristic ULS method
requires only $O(N^2)$ time.  In step 7, BestSubgraph could be called up
to $J$ times for each graph structure, for each of the $R$ randomly
generated sequences of edge removals, resulting in a total of $O(JRN^2)$ calls.
However, BestSubgraph is only called when the removal of an edge $e_{ik}$
disconnects the highest scoring connected subgraph $S_{mj}^\ast$ for that graph
$G_m$ and training example $D_j$.  We now consider the sequence of edge
removals for graphs $G_1 \ldots G_M$, where $M=\frac{N(N-1)}{2}$, and
compute the expected number of calls to BestSubgraph for these $O(N^2)$
edge removals.  We focus on the case of random edge removals, since these
dominate the overall runtime for large $R$.

For a given training example $D_j$, let $x_m$ denote the number of nodes
in the highest-scoring connected subgraph $S_{mj}^\ast$ for graph $G_m$, and
let $T_m$ denote any spanning tree of $S_{mj}^\ast$.  We note that the number
of edges in $T_m$ is $x_m - 1$, which is $O(\min(N,m))$.  Moreover, any
edge that is not in $T_m$ will not disconnect $S_{mj}^\ast$, and thus the
probability of disconnecting $S_{mj}^\ast$ for a random edge removal is upper
bounded by the ratio of the number of disconnecting edges $O(\min(N,m))$
to the total number of edges $m$.  Thus the expected number of calls to
BestSubgraph for graphs $G_1 \ldots G_M$ for the given training example is
$\sum_{m = 1 \ldots M} \frac{O(\min(N,m))}{m}$ = $O(N) + \sum_{m = N
\ldots M} \frac{O(N)}{m}$ = $O(N) + O(N) \sum_{m = N \ldots M}
\frac{1}{m}$ = $O(N \log N)$.  Hence the expected number of calls to
BestSubgraph needed for all $J$ training examples is $O(JN \log N)$ for
the given sequence of graphs $G_1 \ldots G_M$, and $O(JRN \log N)$ for
the $R$ random sequences of edge removals.

Finally, we consider the complexity of choosing the next edge to remove
(step 5 of our graph structure learning framework).  The BestEdge function
is called $O(N^2)$ times for the given sequence of graphs $G_1 \ldots
G_M$, but is not called for the $R$ random sequences of edge removals.
For the GrCorr and PsCorr methods, for each graph $G_m$ and each training
example $D_j$, we must evaluate all $O(m)$ candidate edge removals.  This
requires a total of $O(JN^4)$ checks to determine whether removal of each
edge $e_{ik}$ disconnects the highest scoring connected subgraph $S_{mj}^\ast$
for each graph $G_m$ and training example $D_j$.  The GrCorr method must
also call BestSubgraph whenever the highest scoring subgraph is
disconnected.  However, for a given graph $G_m$ and training example
$D_j$, we show that only $O(N)$ of the $O(m)$ candidate edge removals can
disconnect the highest scoring subset, thus requiring only $O(JN^3)$ calls
to BestSubgraph rather than $O(JN^4)$. To see this, let $x_m$ be the
number of nodes in the highest-scoring connected subgraph $S_{mj}^\ast$, and
let $T_m$ be any spanning tree of $S_{mj}^\ast$.  Then any edge that is not in
$T_m$ will not disconnect $S_{mj}^\ast$, and $T_m$ only has $x_m - 1 = O(N)$ edges.

\subsection{Consistency of Greedy Search}

The greedy algorithm described above is not guaranteed to recover the true graph structure $G_T$.  However, we can show that, given a sufficiently strong and homogeneous signal, and sufficiently
many training examples, the true graph will be part of the sequence of graphs $G_0 \ldots G_M$ identified by the greedy search procedure.  More precisely, let us make assumptions (A1) and (A2)
given in \S3 above.  We also assume that GraphScan (GS) or Upper Level Sets (ULS) is used for BestSubgraph, and that Greedy Correlation (GrCorr) or Pseudo-Greedy
Correlation (PsCorr) is used for selecting the next edge to remove (BestEdge).  Given these assumptions, we can show:

\begin{theorem}
\label{second_theorem}
If the signal is $\alpha_j$-homogeneous and $\frac{\beta_j}{\eta_j}$-strong for all training examples $D_j \sim \mathbf{D}$, and if the set of training examples
$D_1 \ldots D_J$ is sufficiently large, then the true graph $G_T$ will be part of the sequence of
graphs $G_0 \ldots G_M$ identified by Algorithm 1.
\end{theorem}
\begin{proof}

Given an $\alpha_j$-homogeneous and $\frac{\beta_j}{\eta_j}$-strong signal, both GS and ULS will correctly identify the highest-scoring
connected subgraph $S_{mj}^\ast$.  This is true for GS in general, since an exact search is performed, and also true for ULS since $S_{mj}^\ast$ will be one of
the upper level sets considered.  Now let $m_T$ denote the number of edges in the true graph
$G_T$, and consider the sequence of graphs $G_M$, $G_{M-1}$, \ldots, $G_{m_T+1}$ identified by the greedy search procedure. For each of these graphs $G_m$, the next edge to be removed (producing
graph $G_{m-1}$) will be either an edge in $G_T$ or an edge in $G_M \setminus G_T$.  We will show that an edge in $G_M \setminus G_T$ is chosen for removal at each step.  Given assumptions (A1)-(A2) and an $\alpha_j$-homogeneous and $\frac{\beta_j}{\eta_j}$-strong signal, Theorem 1 implies:

a) For any graph that contains all edges of the true graph ($G_T \setminus G_m = \emptyset$), we will have $S_{mj}^\ast = S_j^\ast =
S_j^T$ for all $D_j \sim \mathbf{D}$, and thus $\bar F_{norm}(G_m) = 1$.

b) For any graph that does not contain all edges of the true graph, and for any training example $D_j$ drawn from $\mathbf{D}$, there
is a non-zero probability that we will have $S_{mj}^\ast \ne S_j^\ast$, $F_{mj} < F_j$, and thus $\bar F_{norm}(G_m) < 1$.

We further assume that the set of training examples is sufficiently large so that every pair of nodes $\{v_1,v_2\}$ in $G_T$ is the affected subgraph for at least one training example $D_j$;
note that assumption (A1) ensures that each such pair will be drawn from $\mathbf{D}$ with non-zero probability.  This means that removal of any edge in $G_T$ will disconnect $S_{mj}^\ast$ for
at least one training example $D_j$, leading to $S_{(m-1)j}^\ast \ne S_{mj}^\ast$ and $\bar F_{norm}(G_{m-1}) < \bar F_{norm}(G_m)$, while removal of any edge in $G_M \setminus G_T$ will not
disconnect $S_{mj}^\ast$ for any training examples, maintaining $\bar F_{norm}(G_{m-1}) = \bar F_{norm}(G_m)$.  Hence for both GrCorr, which removes the edge that maximizes $\bar
F_{norm}(G_{m-1})$, and PsCorr, which removes the edge that disconnects $S_{mj}^\ast$ for the fewest training examples, the greedy search procedure will remove all edges in $G_M \setminus G_T$ before removing any edges in $G_T$, leading to $G_{m_T} = G_T$.
\end{proof}

\section{Related Work}

We now briefly discuss several streams of related work.  As noted above,
various spatial scan methods have been proposed for detecting the most
anomalous subset in data with an underlying, known graph structure,
including Upper Level Sets~\citep{patil04}, FlexScan~\citep{tango05}, and
GraphScan~\citep{speakman14}, but none of these methods attempt to learn
an unknown graph structure from data.  Link prediction algorithms such
as~\citep{Taskar:2003:LPR, Vert:2005:SGI} start with an existing network
of edges and attempt to infer additional edges which might also be
present, unlike our scenario which requires inferring the complete edge
structure.  Much work has been done on learning the edge structure of
graphical models such as Bayesian networks and probabilistic relational
models~\citep{Getoor:2003:LPM}, but these methods focus on understanding
the dependencies between multiple attributes rather than learning a
graph structure for event detection. Finally, the recently proposed
NetInf~\citep{Gomez-Redriguez:2010:NetInf},
ConNIe~\citep{Myers:2010:LSNI}, and
MultiTree~\citep{Gomez-Redriguez:2012:MultiTree} methods share our goal
of efficiently learning graph structure.  NetInf is a submodular
approximation algorithm for predicting the latent network structure and
assumes that all connected nodes influence their neighbors with equal
probability. ConNIe relaxes this
assumption and uses convex programming to rapidly infer the optimal
latent network, and MultiTree is an extension of NetInf which considers
all possible tree structures instead of only the most probable ones. The
primary difference of the present work from NetInf, ConNIe, and
MultiTree is that we learn the underlying graph structure from
\emph{unlabeled} data: while these methods are given the affected subset
of nodes for each time step of an event, thus allowing them to learn the
network edges along which the event spreads, we consider the more
difficult case where we are given only the observed and expected counts
at each node, and the affected subset of nodes is not labeled. Further,
these methods are not targeted towards learning a graph structure for
event detection, and we demonstrate below that our approach achieves
more timely and accurate event detection than MultiTree, even when
MultiTree has access to the labels.

\section{Experimental Setup}

In our general framework, we implemented two methods for
$\mathrm{BestSubgraph}(G,D)$: GraphScan (GS) and Upper Level Sets (ULS).
We also implemented three methods for $\mathrm{BestEdge}(G,D)$: GrCorr,
PsCorr, and Corr.  However, using GraphScan with the true greedy method
(GS-GrCorr) was computationally infeasible for our data, requiring 3
hours of run time for a single 50-node graph, and failing to complete
for larger graphs.  Hence our evaluation compares five combinations of
BestSubgraph and BestEdge: GS-PsCorr, GS-Corr, ULS-GrCorr, ULS-PsCorr,
and ULS-Corr.

We compare the performance of our learned graphs with the learned graphs
from MultiTree, which was shown to outperform previously proposed graph
structure learning algorithms such as NetInf and
ConNIe~\citep{Gomez-Redriguez:2012:MultiTree}.  We used the publicly
available implementation of the algorithm, and considered both the case
in which MultiTree is given the true labels of the affected subset of
nodes for each training example (MultiTree-Labels), and the case in
which these labels are not provided (MultiTree-NoLabels).  In the latter
case, we perform a subset scan for each training example $D_j$, and use
the highest-scoring unconstrained subset $S_j^\ast$ as an approximation of
the true affected subset.

\subsection{Description of Data}

Our experiments focus on detection of simulated disease outbreaks
injected into real-world Emergency Department (ED) data from ten
hospitals in Allegheny County, Pennsylvania.  The dataset consists of
the number of ED admissions with respiratory symptoms for each of the
$N=97$ zip codes for each day from January 1, 2004 to December 31, 2005.
The data were cleaned by removing all records where the admission date
was missing or the home zip code was outside the county. The resulting
dataset had a daily mean of 44.0 cases, with a standard deviation of 12.1.

\subsection{Graph-Based Outbreak Simulations}

Our first set of simulations assume that the disease outbreak starts at a randomly
chosen location and spreads over some underlying graph structure,
increasing in size and severity over time.  We assume that an affected
node remains affected through the outbreak duration, as in the
Susceptible-Infected contagion model~\citep{Bailey:1975:InfectionModels}.
For each simulated outbreak, we first choose a center zip code uniformly
at random, then order the other zip codes by graph distance (number of
hops away from the center for the given graph structure), with ties broken
at random.  Each outbreak was assumed to be 14 days in duration.  On each
day $d$ of the outbreak ($d=1 \ldots 14$), we inject counts into the $k$
nearest zip codes, where $k = SpreadRate \times d$, and $SpreadRate$ is a
parameter which determines how quickly the inject spreads.  For each
affected node $v_i$, we increment the observed count $c_i^t$ by
$\mbox{Poisson}(\lambda_i^t)$, where $\lambda_i^t = \frac{SpreadFactor
\times d}{SpreadFactor+\log (dist_i+1)}$, and $SpreadFactor$ is a
parameter which determines how quickly the inject severity decreases with
distance.  The assumption of Poisson counts is common in
epidemiological models of disease spread; the expected number of
injected cases $\lambda_i^t$ is an increasing function of the inject day
$d$, and a decreasing function of the graph distance between the affected
node and the center of the outbreak.  We considered 4 different inject
types, as described below; for each type, we generated $J = 200$ training
injects (for learning graph structure) and an additional 200 test injects
to evaluate the timeliness and accuracy of event detection given the
learned graph.

\subsubsection{Zip code adjacency graph based injects}

We first considered simulated outbreaks which spread from a given zip code
to spatially adjacent zip codes, as is commonly assumed in the literature.
Thus we formed the \emph{adjacency graph} for the 97 Allegheny County zip
codes, where two nodes are connected by an edge if the corresponding zip
codes share a boundary.  We performed two sets of experiments: for the
first set, we generated simulated injects using the adjacency graph, while
for the second set, we added additional edges between randomly chosen
nodes to simulate travel patterns.  As noted above, a contagious disease
outbreak might be likely to propagate from one location to another
location which is not spatially adjacent, based on individuals' daily
travel, such as commuting to work or school.  We hypothesize that
inferring these additional edges will lead to improved detection
performance.

\ignore{\subsubsection{Spatial injects}

An alternative assumption is that a disease outbreak affects all areas
within some spatial radius of the outbreak center. This is commonly
assumed in the literature for non-contagious outbreaks with a point
source, such as cancer clusters resulting from a radiation leak, or other
patterns of illness resulting from environmental
exposures~\citep{kulldorff97a}.  To simulate these ``spatial injects'', we
do not assume an underlying graph structure, but instead assume that an
outbreak spreads to the $k$-nearest neighbors of the center zip code based
on Euclidean distance between the zip code centroids, as described above.}

\subsubsection{Random graph based injects}

Further, in order to show that we can learn a diverse set of graph structures
over which an event spreads, we performed experiments assuming two types
of random graphs, Erdos-Renyi and preferential attachment.  For each
experiment, we used the same set of nodes $V$ consisting of the 97
Allegheny County zip codes, but created a random set of edges $E$ connecting
these nodes; the graph $G=(V,E)$ was then used to simulate 200 training
and 200 test outbreaks, with results averaged over multiple such
randomly chosen graphs.

First, we considered
\emph{Erdos-Renyi graphs} (assuming that each pair of nodes is connected
with a constant probability $p$), with edge probabilities $p$ ranging from
0.08 to 0.20.  The relative performance of methods was very
similar across different $p$ values, and thus only the averaged results
are reported.  Second, we considered \emph{preferential attachment
graphs}, scale-free network graphs which are constructed by adding nodes
sequentially, assuming that each new node forms an edge to each existing
node with probability proportional to that node's degree.  We generated
the preferential attachment graph by first connecting three randomly
chosen nodes, then adding the remaining nodes in a random order. Each
new node that arrives attaches itself to each existing node $v_j$ with
probability $\frac{deg(v_j)}{\sum_i deg(v_i)}$, where each node's maximum
degree was restricted to $0.2 \times |V|$.

\ignore{Finally, we
generated random graphs spanning the continuum between
preferential attachment and Erdos-Renyi graphs.  To do so, we follow the
same sequential procedure as for the preferential attachment graphs, but
each new node that arrives attaches itself to each existing node $v_j$
with probability $\alpha \frac{deg(v_j)}{\sum_i
deg(v_i)} + (1-\alpha)p$.  The parameter $\alpha$ controls the
level of preferential attachment, with $\alpha = 1$ corresponding to a
preferential attachment graph and  $\alpha = 0$ corresponding
to an Erdos-Renyi graph with edge probability $p$.}

\subsection{Simulated Anthrax Bio-Attacks}

We present additional evaluation results for one potentially realistic
outbreak scenario, an increase in respiratory Emergency Department cases
resulting from an airborne release of anthrax spores (e.g.~from a
bio-terrorist attack).  The anthrax attacks are based on a state-of-the-art,
highly realistic simulation of an aerosolized anthrax release, the
Bayesian Aerosol Release Detector (BARD) simulator ~\citep{hogan07}.  BARD
uses a combination of a dispersion model (to determine which areas will
be affected and how many spores people in these areas will be exposed to),
an infection model (to determine who will become ill with anthrax and
visit their local Emergency Department),and a visit delay model to calculate
the probability of the observed Emergency Department visit counts over
a spatial region.  These complex simulations take into account weather data
when creating the affected zip codes and demographic information
when calculating the number of additional Emergency Department cases within
each affected zip code.  The weather patterns are modeled with Gaussian
plumes resulting in elongated, non-circular regions of affected zip codes.  Wind
direction, wind speed, and atmospheric stability all influence the shape
and size of the affected area.  A total of 82 simulated anthrax attacks were
generated and injected into the Allegheny County Emergency Department
data, using the BARD model.  Each simulation generated
between 33 and 1324 cases in total (mean = 429.2, median = 430) over a
ten-day outbreak period; half of the attacks were used for training and
half for testing.

\section{Experimental Results}

\subsection{Computation Time}

For each of the experiments described above (adjacency, adjacency plus travel
patterns, Erdos-Renyi random graphs, and preferential attachment graphs), we report
the average computation time required for each of our methods (Table~1).
Randomization testing is not included in these results, since it is not dependent on
the choice of BestEdge.  Each sequence of randomized edge removals $G_{1,r}, \ldots, G_{M,r}$
required 1 to 2
hours for the GraphScan-based methods and 1 to 3 minutes for the ULS-based methods.
\begin{table} \begin{center}
\caption{Average run time in minutes for each learned graph structure, for $N=97$
nodes.}\vspace{1mm}
{\tiny
\begin{tabular}{|c|c|c|c|c|c|c|c|}\hline
Experiment & \multicolumn{2}{c|}{GraphScan (GS)} & \multicolumn{3}{c|}{ULS}
& \multicolumn{2}{c|}{MultiTree}\\
& \multicolumn{1}{c}{PsCorr}& Corr & \multicolumn{1}{|c}{GrCorr} &
\multicolumn{1}{c}{PsCorr} & Corr &\multicolumn{1}{c}{Labels} & NoLabels\\\hline
Adjacency & 41 & 38 & 13 & 2 & 1 & $<$1 & $<$1\\
Adjacency+Travel & 53 & 47 & 15 & 3 & 1 & $<$1 & $<$1\\
Erdos-Renyi (avg) & 93 & 89 & 22 & 6 & 3 & $<$1 & $<$1\\
Pref. Attachment & 49 & 44 & 17 & 3 & 1 & $<$1 & $<$1\\
\hline \end{tabular}
}
\vspace{-0.4cm}
\end{center} \end{table}
\begin{table} \begin{center} \caption{Average run time in minutes for each learned
graph structure, for Erdos-Renyi graphs with varying numbers of nodes $N$.}
{\tiny
\begin{tabular}{|c|c|c|c|c|c|c|c|}\hline Size & \multicolumn{2}{c|}{GraphScan (GS)} & \multicolumn{3}{c|}{ULS}
& \multicolumn{2}{c|}{MultiTree}\\
& \multicolumn{1}{c}{PsCorr}& Corr & \multicolumn{1}{|c}{GrCorr} &
\multicolumn{1}{c}{PsCorr} & Corr &\multicolumn{1}{c}{Labels} & NoLabels\\\hline
N=50 & 2 & 2 & 1 & $<$1 & $<$1 & $<$1 & $<$1\\
N=75 & 37 & 32 & 3 & 1 & $<$1 & $<$1 & $<$1\\
N=100 & 58 & 53 & 13 & 3 & $<$1 & $<$1 & $<$1\\
N=200 & - & - & 91 & 33 & 1 & 1 & 1\\
N=500 & - & - & 2958 & 871 & 27 & 2 & 2\\
\hline \end{tabular}
}
\vspace{-0.4cm}
\end{center} \end{table}

For each of the $J=200$ training examples, all methods except for ULS-GrCorr
required fewer than 80 calls to BestSubgraph on average to search over the space of
$M=4,656$ graph structures, a reduction of nearly two orders of magnitude as
compared to the naive approach of calling BestSubgraph for each combination of graph
structure and training example.  Similarly, a naive implementation of the true
greedy search would require approximately 11 million calls to BestSubgraph for each
training example, while our ULS-GrCorr approach required only $\sim$5000 calls per
training example, a three order of magnitude speedup.  As expected, ULS-Corr and
ULS-PsCorr had substantially faster run times than GS-Corr and GS-PsCorr, though the
GraphScan-based approaches were still able to learn each graph structure in less
than two hours.

Next, in order to evaluate how each method scales with the number of nodes $N$, we
generated Erdos-Renyi random graphs with edge probability $p=0.1$ and $N$ ranging
from 50 to 500.  For each graph, we generated simulated counts and baselines, as
well as simulating injects to produce $J=200$ training examples for learning the
graph structure.  Table 2 shows the average time in minutes required by each method
to learn the graph structure.  We observe that the ULS-based methods were
substantially faster than the GraphScan-based methods, and were able to scale to
graphs with $N=500$ nodes, while GS-Corr and GS-PsCorr were not computationally
feasible for $N \ge 200$.  We note that MultiTree has much lower computation
time as compared to our graph learning methods, since it is not dependent on calls
to a graph-based event detection method (BestSubgraph); however, its detection performance
is lower, as shown below in our experiments.

\subsection{Comparison of True and Learned Graphs}

For each of the four graph-based injects (adjacency, adjacency plus
travel patterns, Erdos-Renyi, and preferential attachment), we compare
the learned graphs to the true underlying graph over which the simulated
injects spread.  Table 3 compares the number of edges in the true
underlying graph to the number of edges in the learned graph structure
for each of the methods, and Tables 4 and 5 show the precision and recall of
the learned graph as compared to the true graph.  Given the true set of
edges $E^T$ and the learned set of edges $E^\ast$, the edge
precision and recall are defined to be $\frac{|E^\ast \cap
E^T|}{|E^\ast|}$ and $\frac{|E^\ast \cap E^T|}{|E^T|}$
respectively.  High recall means that the learned graph structure
identifies a high proportion of the true edges, while high precision
means that the learned graph does not contain too many irrelevant edges.
We observe that GS-PsCorr had the highest recall, with nearly identical
precision to GS-Corr and ULS-GrCorr.  MultiTree had higher precision and
comparable recall to GS-PsCorr when it was given the true labels, but 3-5\%
lower precision and recall when the labels were not provided.
\begin{table}[t]
\begin{center} \caption{Comparison of true and learned number of
edges $m$.}
{\tiny
\begin{tabular}{|c|c|c|c|c|c|c|c|c|}\hline
Experiment & Edges & \multicolumn{7}{c|}{Learned Edges} \\
& (true) & \multicolumn{2}{c|}{GraphScan (GS)} &
\multicolumn{3}{c|}{ULS} & \multicolumn{2}{c|}{MultiTree}\\
& & \multicolumn{1}{c}{PsCorr}& Corr & \multicolumn{1}{|c}{GrCorr} &
\multicolumn{1}{c}{PsCorr} & Corr &\multicolumn{1}{c}{Labels} & NoLabels\\\hline
Adjacency & 216 & 319 & 297 & 305 & 332& 351 & 280& 308\\
Adjacency+Travel & 280 & 342 & 324 & 329 & 362& 381 & 316& 342\\
Erdos-Renyi ($p=0.08$) & 316 & 388 & 369 & 359& 398& 412& 356& 382\\
Pref. Attachment & 374 & 394 & 415 & 401 & 428& 461& 399& 416 \\ \hline
\end{tabular}
}
\vspace{-0.4cm}
\end{center}
\label{edges_table}
\end{table}
\begin{table}[H]
\begin{center} \caption{Comparison of edge precision for learned graphs.}
{\tiny
\begin{tabular}{|c|c|c|c|c|c|c|c|}\hline
Experiment & \multicolumn{7}{c|}{Precision} \\
& \multicolumn{2}{c|}{GraphScan (GS)} & \multicolumn{3}{c|}{ULS}
&\multicolumn{2}{c|}{MultiTree} \\
& \multicolumn{1}{c}{PsCorr} & \multicolumn{1}{c|}{Corr}
& \multicolumn{1}{c}{GrCorr} & \multicolumn{1}{c}{PsCorr} &
\multicolumn{1}{c|}{Corr}
& \multicolumn{1}{c}{Labels} & NoLabels \\ \hline
Adjacency & 0.60 & 0.62 & 0.62 & 0.53 & 0.50& 0.66 & 0.58\\
Adjacency+Travel & 0.70 & 0.71 & 0.69 & 0.60& 0.52 & 0.75& 0.65 \\
Erdos-Renyi (avg) & 0.56 & 0.59 & 0.61& 0.59& 0.54 & 0.62& 0.56 \\
Pref. Attachment & 0.83 & 0.79 & 0.80 & 0.69 & 0.59 & 0.86& 0.80 \\ \hline
\end{tabular}
}
\vspace{-0.4cm}
\end{center}
\label{edge_prec_table}
\end{table}
\begin{table}[H]
\begin{center} \caption{Comparison of edge recall for learned graphs.}
{\tiny
\begin{tabular}{|c|c|c|c|c|c|c|c|}\hline
Experiment & \multicolumn{7}{c|}{Recall} \\
& \multicolumn{2}{c|}{GraphScan (GS)} & \multicolumn{3}{c|}{ULS}
&\multicolumn{2}{c|}{MultiTree} \\
& \multicolumn{1}{c}{PsCorr} & \multicolumn{1}{c|}{Corr}
& \multicolumn{1}{c}{GrCorr} & \multicolumn{1}{c}{PsCorr} &
\multicolumn{1}{c|}{Corr}
& \multicolumn{1}{c}{Labels} & NoLabels \\ \hline
Adjacency & 0.89 & 0.86 & 0.88 & 0.81&0.77 & 0.86& 0.83\\
Adjacency+Travel & 0.86 & 0.83 & 0.81 & 0.77 &0.71 & 0.85& 0.79\\
Erdos-Renyi (avg) & 0.87 & 0.81 & 0.83 & 0.79& 0.70 & 0.84& 0.79\\
Pref. Attachment &0.88& 0.81 & 0.86 & 0.79 & 0.73 & 0.91 & 0.89\\ \hline
\end{tabular}
}
\vspace{-0.4cm}
\end{center}
\label{edge_recall_table}
\end{table}

\subsection{Comparison of Detection Performance}
\label{DetPerf}

We now compare the detection performance of the learned graphs on the test data: a
separate set of 200 simulated injects (or $41$ injects for the BARD anthrax simulations), generated from the same distribution as the
training injects which were used to learn that graph. To evaluate a graph, we use
the GraphScan algorithm (assuming the given graph structure) to identify the
highest-scoring connected subgraph $S$ and its likelihood ratio score $F(S)$ for
each day of each simulated inject, and for each day of the original Emergency
Department data with no cases injected.  We note that performance was substantially
improved by using GraphScan for detection as compared to ULS, regardless of whether
GraphScan or ULS was used to learn the graph, and GraphScan required less than a few
seconds of run time for detection per day of the ED data.

We then evaluate detection performance using two metrics: average time
to detection (assuming a false positive rate of 1 fp/month, typically
considered acceptable by public health), and spatial accuracy (overlap
between true and detected clusters).  To compute detection time, we
first compute the score threshold $F_{thresh}$ for detection at 1
fp/month.  This corresponds to the 96.7th percentile of the daily scores
from the original ED data.  Then for each simulated inject, we compute
the first outbreak day $d$ with $F(S) > F_{thresh}$; for this computation,
undetected outbreaks are counted as $14$ days (maximum number of inject days)
to detect.  We then average
the time to detection over all 200 test injects.  To evaluate spatial
accuracy, we compute the average overlap coefficient between the
detected subset of nodes $S^\ast$ and the true affected subset
$S^T$ at the midpoint (day 7) of the outbreak, where overlap is
defined as $\frac{|S^\ast \cap S^T|}{|S^\ast \cup S^T|}$.

As noted above, detection performance is often improved by including a proximity constraint, where we perform separate searches over the ``local neighborhood'' of each of the $N$
graph nodes, consisting of that node and its $k-1$ nearest neighbors, and report the highest-scoring connected subgraph over all neighborhoods.  We compare the detection
performance of each graph structure by running GraphScan with varying neighborhood sizes $k=5,10,\ldots,45$ for each outbreak type.

\subsubsection{Results on zip code adjacency graphs}

We first evaluate the detection time and spatial accuracy of GraphScan,
using the learned graphs, for simulated injects which spread based on
the adjacency graph formed from the 97 Allegheny County zip codes, as
shown in Figure~\ref{AdjGraphInject}.  This figure also shows the
performance of GraphScan given the true zip code adjacency graph.  We
observe that the graphs learned by GS-PsCorr and ULS-GrCorr have similar
spatial accuracy to the true zip code adjacency graph, as measured by
the overlap coefficient between the true and detected subsets of nodes,
while the graphs learned by GS-Corr and MultiTree have lower spatial
accuracy.  Surprisingly, all of the learned graphs achieve more timely
detection than the true graph: for the optimal neighborhood size of
$k=30$, ULS-GrCorr and GS-PsCorr detected an average of 1.4 days faster
than the true graph.  This may be because the learned graphs, in
addition to recovering most of the edges of the adjacency graph, also
include additional edges to nearby but not spatially adjacent nodes
(e.g.~neighbors of neighbors).  These extra edges provide added
flexibility to consider subgraphs which would be almost but not quite
connected given the true graph structure.  This can improve detection
time when some nodes are more strongly affected than others, enabling
the strongly affected nodes to be detected earlier in the outbreak
before the entire affected subgraph is identified.  Finally, ULS-GrCorr
and GS-PsCorr detected 0.6 days faster than MultiTree for $k=30$.

 \begin{figure}[t]
 \begin{center}
  \begin{tabular}{cc}
   \includegraphics[scale=0.23]{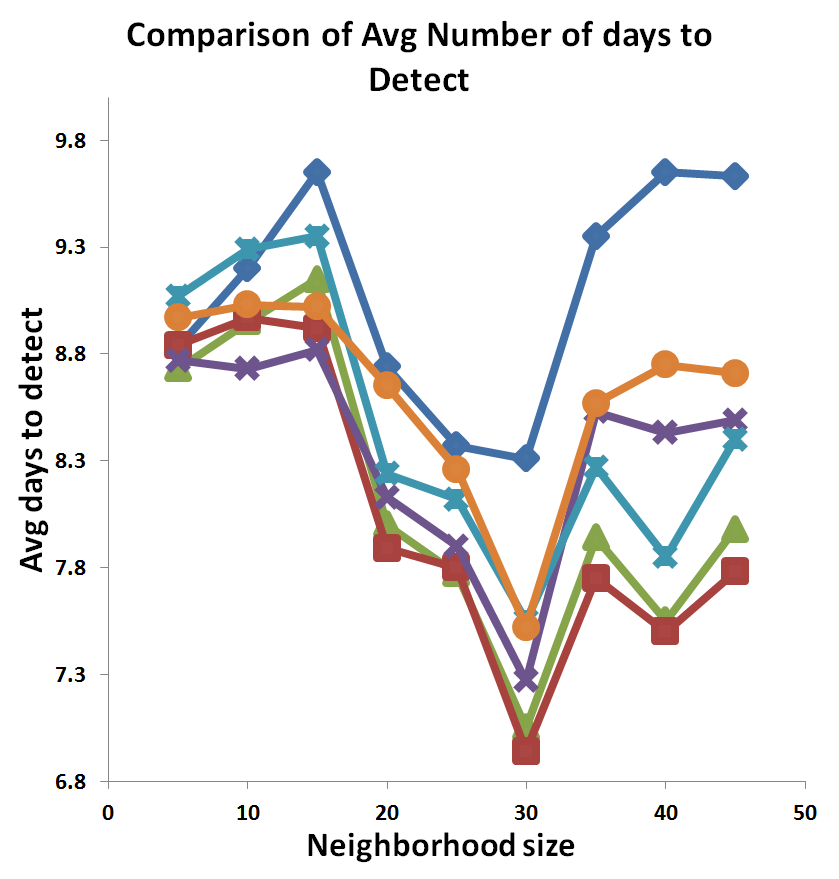}&
   \includegraphics[scale=0.23]{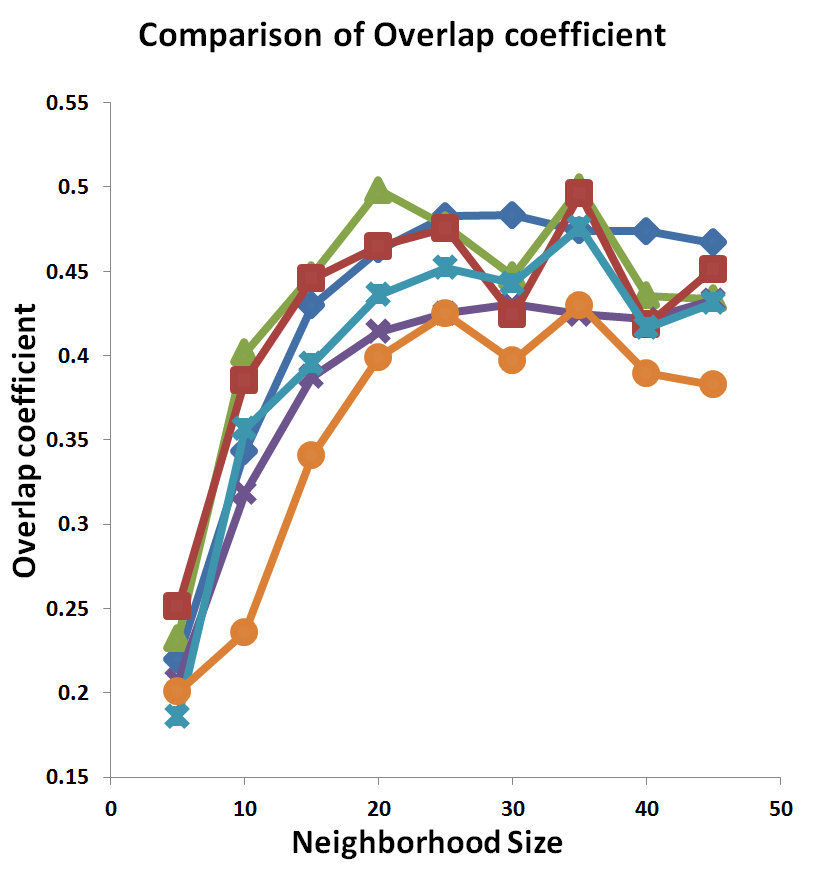}\\
   \multicolumn{2}{c}{\includegraphics[scale=0.70]{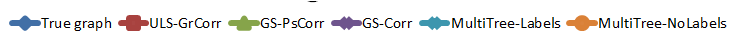}}
   \end{tabular}
   \caption{Comparison of detection performance of the true and learned
graphs for injects based on zip code adjacency.}\label{AdjGraphInject}
 \end{center}
 \vspace{-0.4cm}
 \end{figure}
\begin{figure}[t]
 \begin{center}
 \begin{tabular}{cc}
   \includegraphics[scale=0.23]{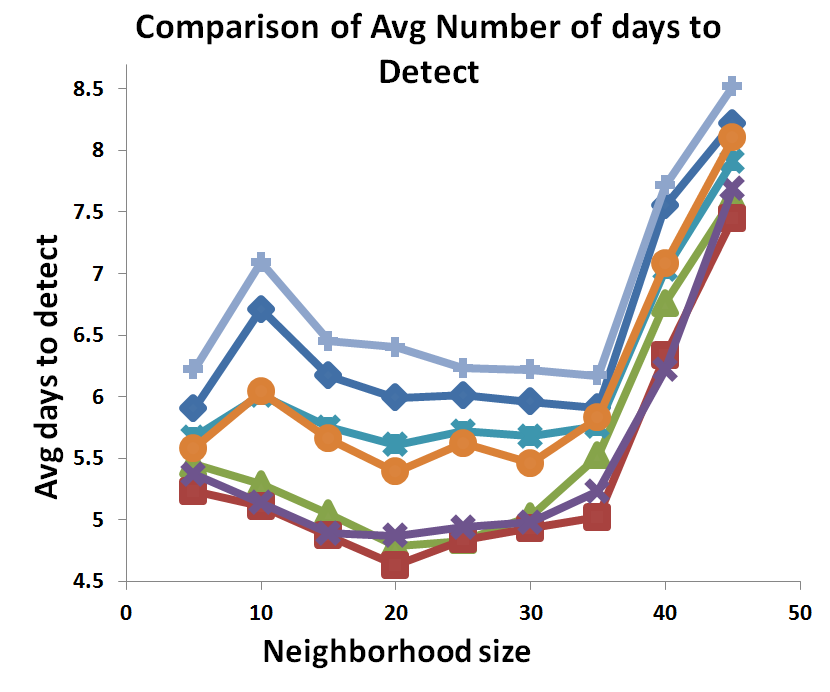} &
   \includegraphics[scale=0.23]{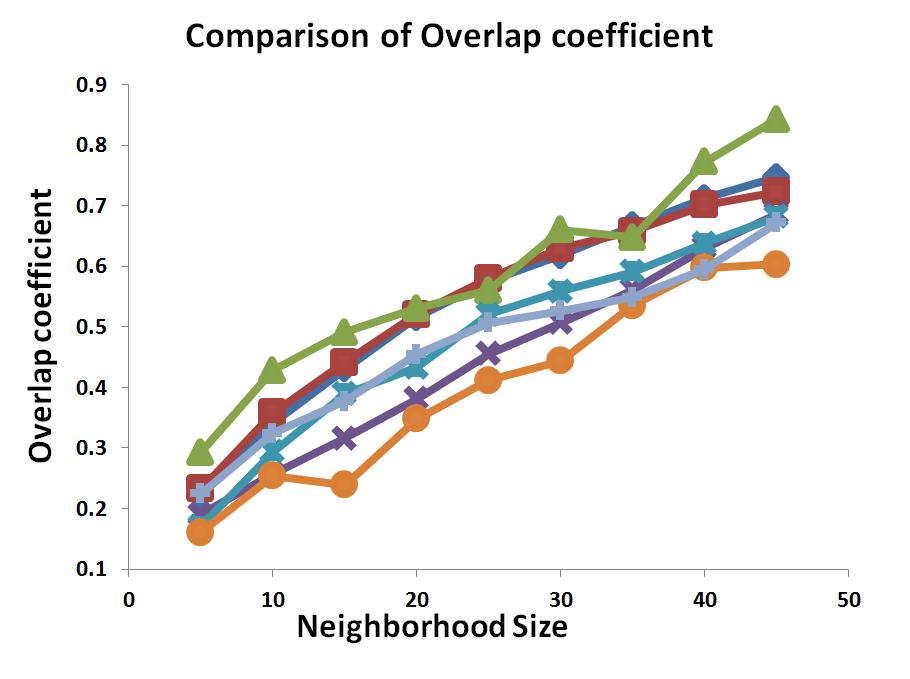}\\
   \multicolumn{2}{c}{\includegraphics[scale=0.60]{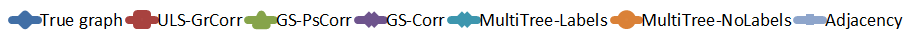}}
 \end{tabular}
   \caption{Comparison of detection performance of the true,
learned, and adjacency graphs for injects based on adjacency with
simulated travel patterns.}\label{AdjRandInject}
 \end{center}
 \vspace{-0.4cm}
 \end{figure}

\subsubsection{Results on adjacency graphs with simulated travel patterns}

Next we compared detection time and spatial accuracy, using the graphs
learned by each of the methods, for simulated injects which spread based
on the zip code adjacency graph with additional random edges added to
simulate travel patterns, as shown in Figure~\ref{AdjRandInject}.  This
figure also shows the detection performance given the true (adjacency
plus travel) graph and the adjacency graph without travel patterns.  We
observe again that GS-PsCorr and ULS-GrCorr achieve similar spatial
accuracy to the true graph, while the original adjacency graph, GS-Corr,
and MultiTree have lower spatial accuracy. Our learned graphs are able
to detect outbreaks 0.8 days earlier than MultiTree, 1.2 days earlier
than the true graph, and 1.7 days earlier than the adjacency graph
without travel patterns.  This demonstrates that our methods can
successfully learn the additional edges due to travel patterns,
substantially improving detection performance.

\subsubsection{Results on random graphs}

 \begin{figure}[t]
 \begin{center}
 \begin{tabular}{cc}
   \includegraphics[scale=0.23]{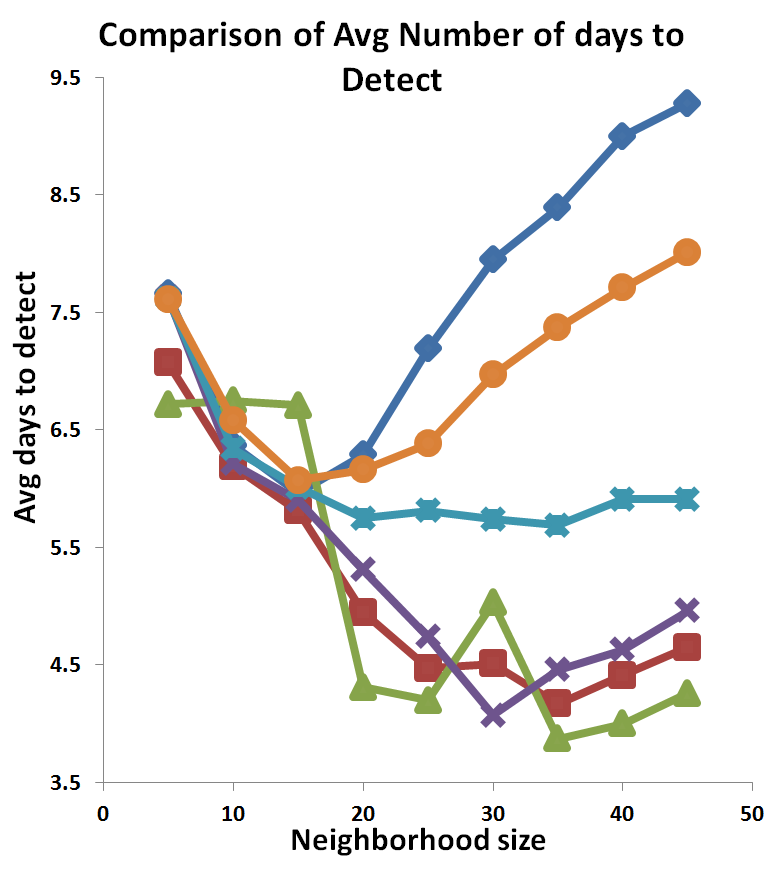} &
   \includegraphics[scale=0.23]{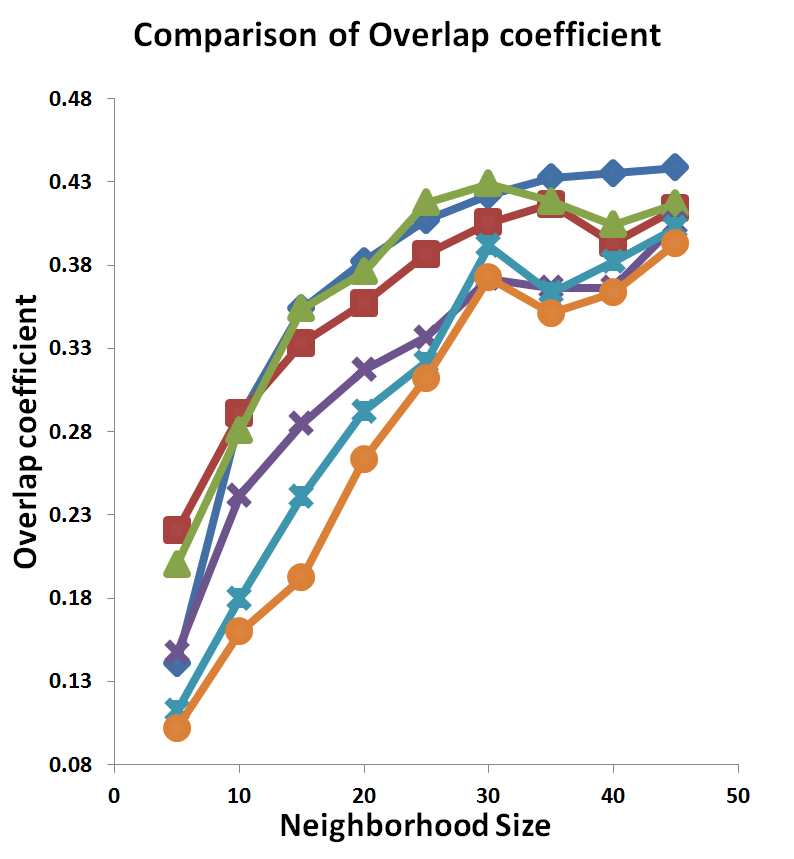}\\
   \multicolumn{2}{c}{\includegraphics[scale=0.70]{Legend_new1.png}}
 \end{tabular}
   \caption{Comparison of detection performance of the true and
learned graphs averaged over seven inject types ($p=0.08,\ldots, 0.20$)
based on Erdos-Renyi random graphs.}\label{RandGraphInject}
 \end{center}
 \vspace{-0.4cm}
 \end{figure}
 \begin{figure}[t]
 \begin{center}
  \begin{tabular}{cc}
   \includegraphics[scale=0.23]{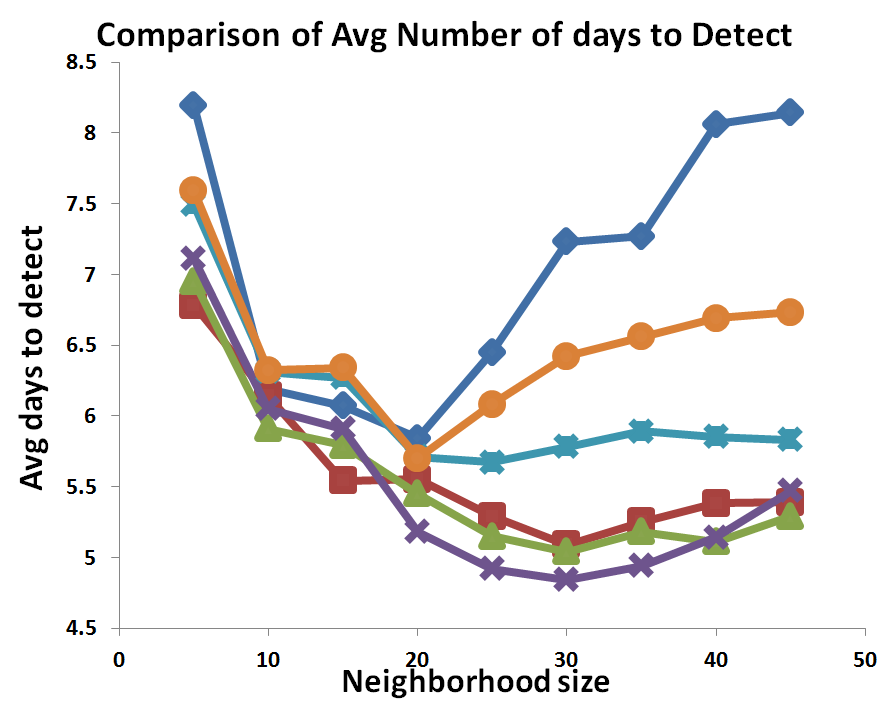} &
   \includegraphics[scale=0.23]{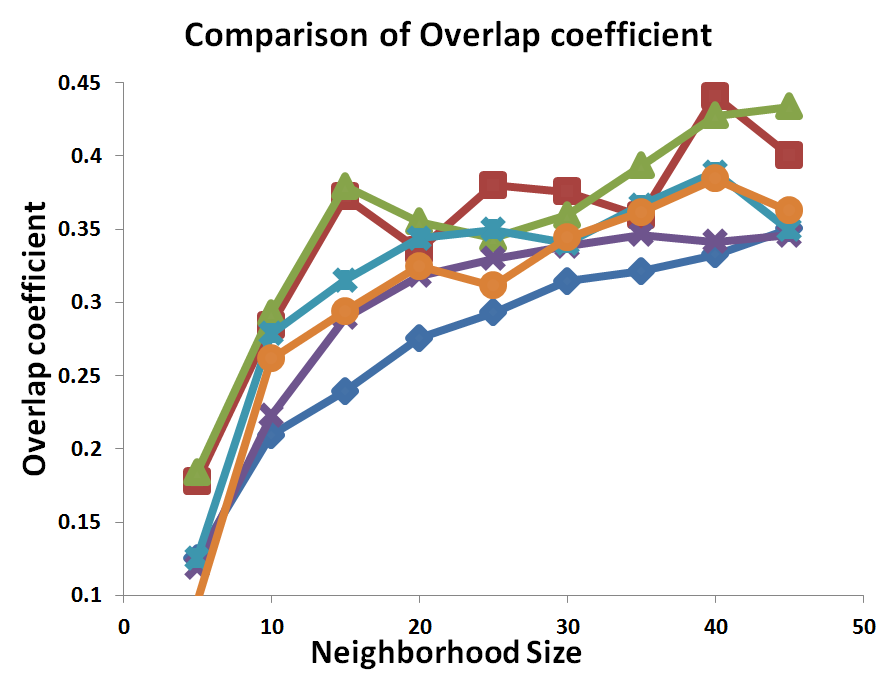}\\
   \multicolumn{2}{c}{\includegraphics[scale=0.70]{Legend_new1.png}}
  \end{tabular}
   \caption{Comparison of detection performance of the true and
learned graphs for injects based on a preferential attachment
graph.}\label{PrefGraphInject}
 \end{center}
 \vspace{-0.4cm}
 \end{figure}

Next we compared detection time and spatial accuracy using the learned
graphs for simulated injects which spread based on Erdos-Renyi and
preferential attachment graphs, as shown in
Figures~\ref{RandGraphInject} and~\ref{PrefGraphInject} respectively.  Each
figure also shows the performance of the true randomly generated graph.
As in the previous experiments, we observe that our learned graphs
achieve substantially faster detection than the true graph and
MultiTree.  For preferential attachment, the learned graphs also achieve
higher spatial accuracy than the true graph, with GS-PsCorr and
ULS-GrCorr again outperforming GS-Corr and MultiTree.  For Erdos-Renyi,
GS-PsCorr and ULS-GrCorr achieve similar spatial accuracy to the true
graph, while GS-Corr and MultiTree have lower accuracy.

\subsubsection{Results on BARD simulations}

We further compared the detection time and spatial accuracy using
learned graphs based on realistic simulations of anthrax bio-attacks, as
shown in Figure~\ref{BARDInject}.  In these simulations there is no ``true''
graph structure as these were generated using spatial information based on environmental
characteristics (wind direction, etc.).  Hence, we compare the performance of various graphs
learned or assumed.  It can be seen that the learned
graphs using GS-PsCorr and ULS-GrCorr achieve substantially faster
detection and higher spatial accuracy, as compared to assuming the adjacency
graph and the graphs learned using GS-Corr and MultiTree.

 \begin{figure}[t]
 \begin{center}
  \begin{tabular}{cc}
   \includegraphics[scale=0.23]{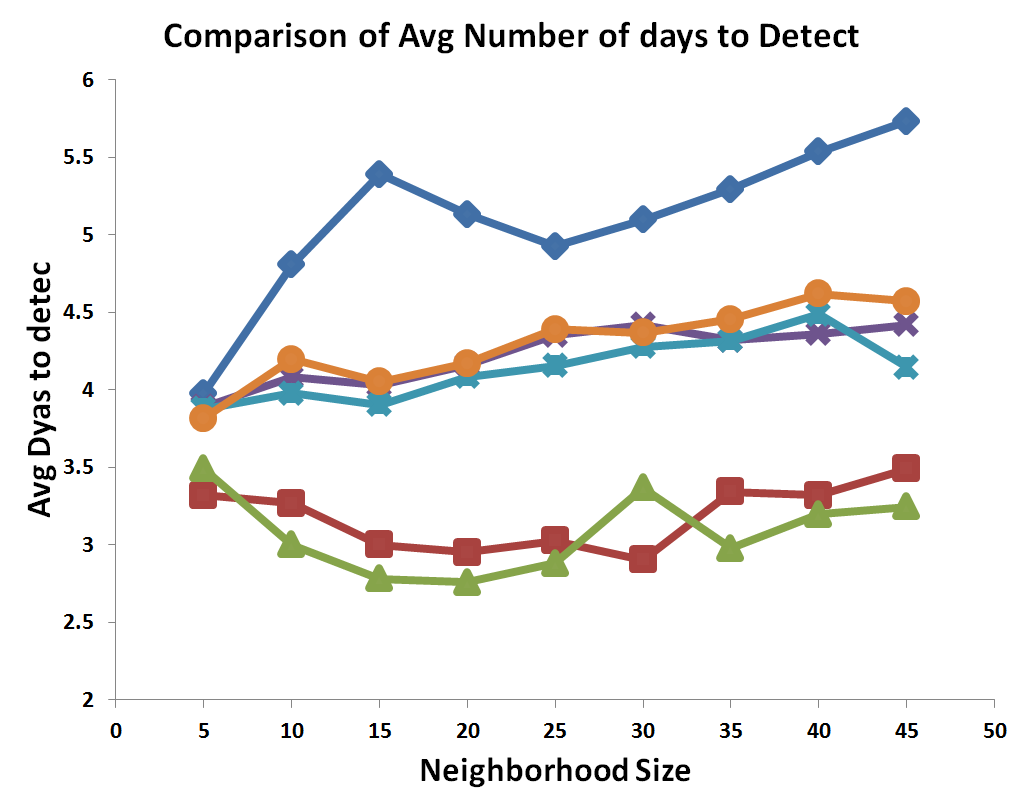} &
   \includegraphics[scale=0.23]{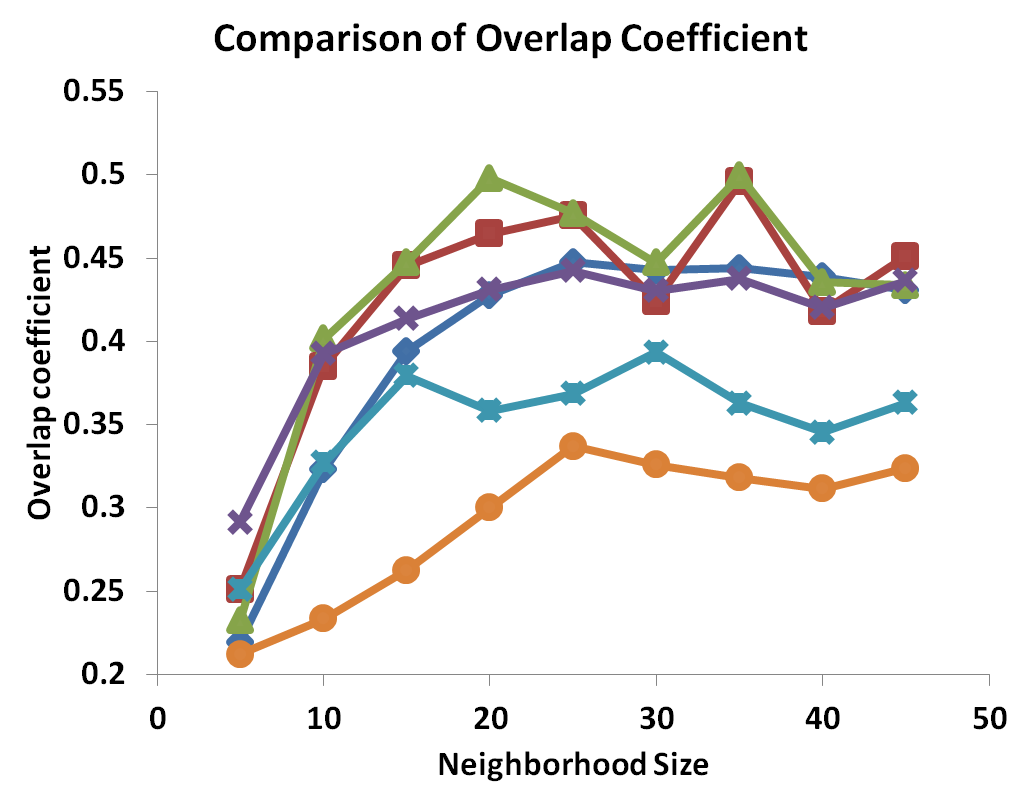}\\
   \multicolumn{2}{c}{\includegraphics[scale=0.50]{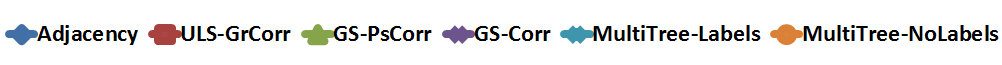}}
  \end{tabular}
   \caption{Comparison of detection performance of the true and
learned graphs for injects based on simulated anthrax bio-attacks.}\label{BARDInject}
 \end{center}
 \vspace{-0.4cm}
 \end{figure}

\subsection{Effect of number of training examples on performance}
\label{num_train_performance}
All of the experiments discussed above (except for the BARD simulations) assume
$J=200$ unlabeled training
examples for learning the graph structure.  We now evaluate the graphs learned by
two of our best performing methods, GS-PsCorr and ULS-GrCorr, using smaller
numbers of training examples ranging from $J=20$ to $J=200$.  Simulated outbreaks
were generated based on the preferential attachment graph described in \S6.2.2.
As shown in Figure~\ref{NumExampAffectPerf}, GS-PsCorr and ULS-GrCorr perform very
similarly both in terms of average number of days to detect and spatial accuracy.
Performance of both methods improves with increasing training set size,
outperforming the true graph structure for $J > 60$.
\begin{figure}[t]
 \begin{center}
  \begin{tabular}{cc}
   \includegraphics[scale=0.2]{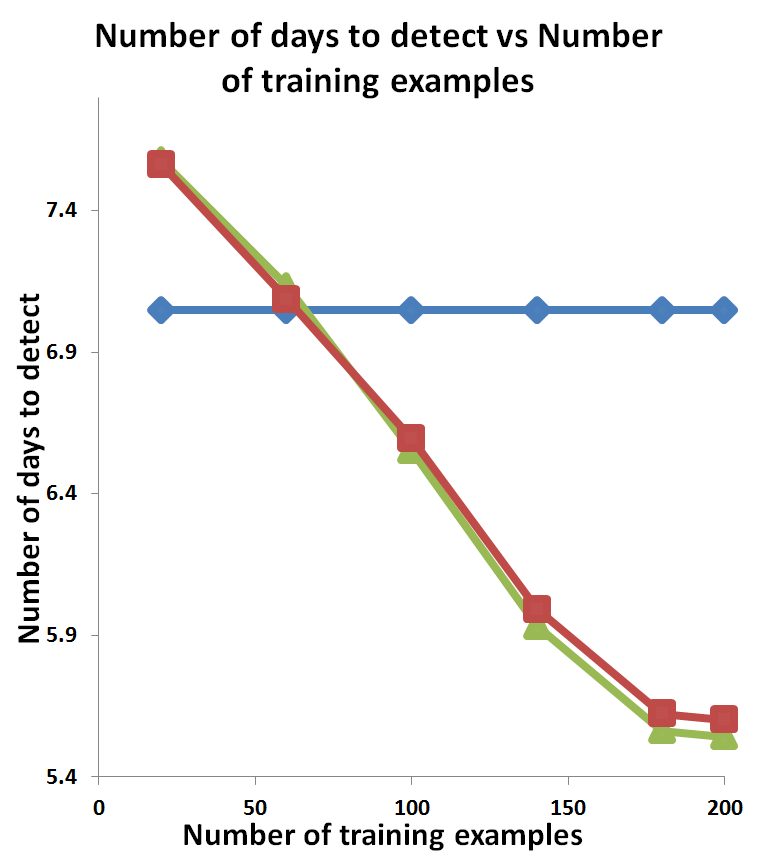} &
   \includegraphics[scale=0.2]{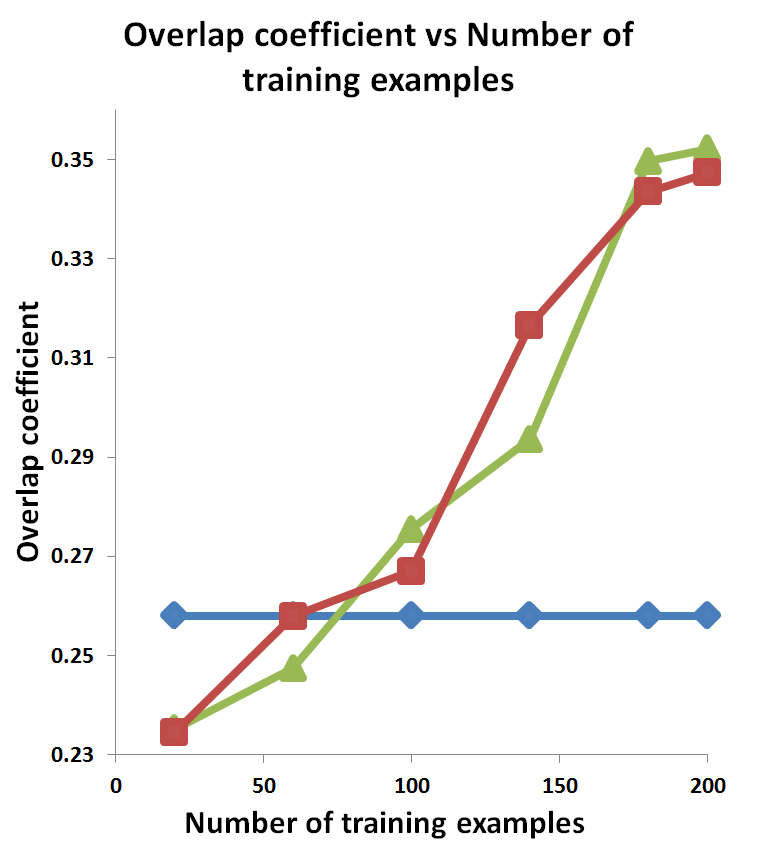}\\
   \multicolumn{2}{c}{\includegraphics[scale=0.3]{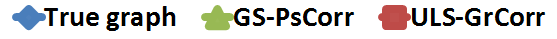}}
  \end{tabular}
   \caption{Effect of number of training examples on performance of GS-PsCorr
   and ULS-GrCorr.}\label{NumExampAffectPerf}
 \end{center}
 \vspace{-0.4cm}
 \end{figure}

\subsection{Effect of percentage of injects in training data on
performance}

All of the experiments discussed above (except for the BARD simulations) assume
that the $J$ unlabeled training
examples are each a ``snapshot'' of the observed count data $c_i^t$ at each node
$v_i$ during a time when an event is assumed to be occurring. However, in practice
the training data may be \emph{noisy}, in the sense that some fraction of the
training examples may be from time periods where no events are present.  Thus we
evaluate performance of the graphs learned by GS-PsCorr and ULS-GrCorr (for
simulated outbreaks based on the preferential attachment graph described in \S6.2.2)
using a set of $J=200$ training examples, where proportion $p$ of the examples are
based on simulated inject data, and proportion $1-p$ are drawn from the original
Emergency Department data with no outbreaks injected.  As shown in
Figure~\ref{PercAffectPerf}, the performance of both GS-PsCorr and ULS-GrCorr
improves as the proportion of injects $p$ in the training data increases.  For $p
\ge 0.6$, both methods achieve more timely detection than the true underlying graph,
with higher spatial accuracy. These results demonstrate that our graph structure
learning methods, while assuming that all training examples contain
true events, are robust to violations of this assumption.

\begin{figure}[t]
 \begin{center}
  \begin{tabular}{cc}
   \includegraphics[scale=0.2]{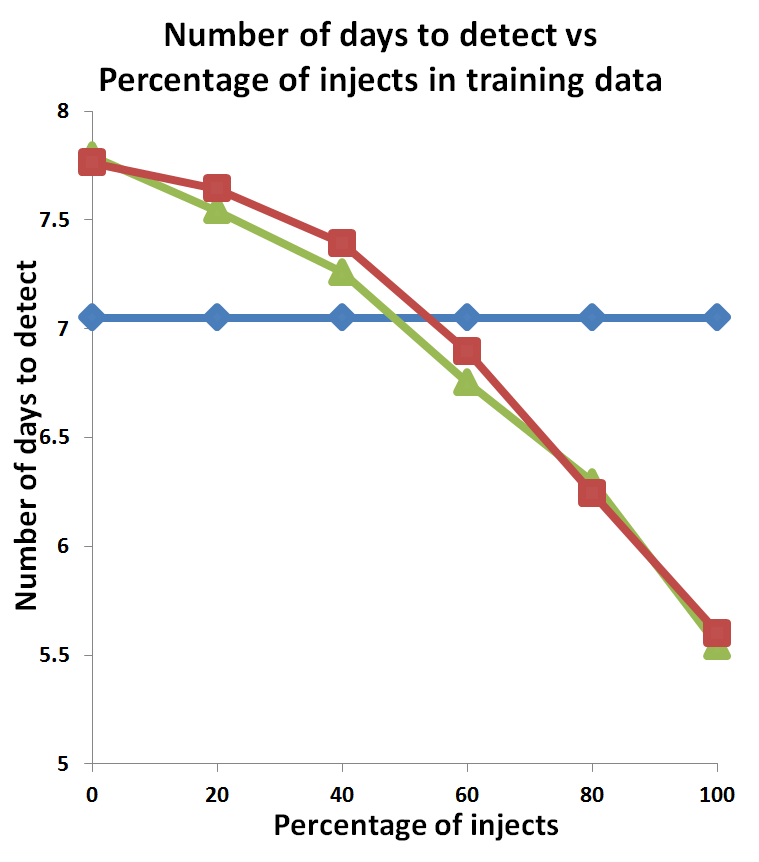}&
   \includegraphics[scale=0.2]{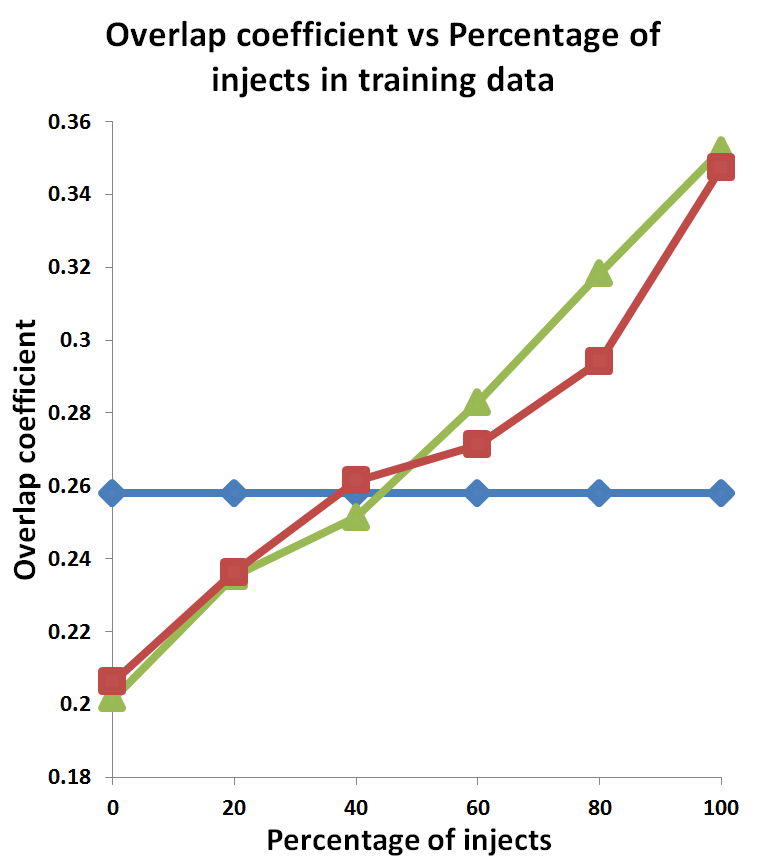}\\
   \multicolumn{2}{c}{\includegraphics[scale=0.3]{Legend_3.png}}
  \end{tabular}
   \caption{Effect of percentage of injects in training data on
performance of GS-PsCorr and ULS-GrCorr learned graphs.}\label{PercAffectPerf}
 \end{center}
 \vspace{-0.4cm}
 \end{figure}

\section{Conclusions and Future Work}

In this work, we proposed a novel framework to learn graph structure from unlabeled
data, based on comparing the most anomalous subsets detected with and without the
graph constraints.  This approach can accurately and efficiently learn a graph
structure which can then be used by graph-based event detection methods such as
GraphScan, enabling more timely and more accurate detection of events (such as
disease outbreaks) which spread based on that latent structure.  Within our general
framework for graph structure learning, we compared five approaches which differed
both in the underlying detection method (BestSubgraph) and the method used to choose
the next edge for removal (BestEdge), incorporated into a provably efficient greedy
search procedure.  We demonstrated both theoretically and empirically that our
framework requires fewer calls to BestSubgraph than a naive greedy approach,
$O(N^3)$ as compared to $O(N^4)$ for exact greedy search, and $O(N\log N)$ as
compared to $O(N^2)$ for approximate greedy search, resulting in 2 to 3 orders
of magnitude speedup in practice.

We tested these approaches on various types of simulated disease
outbreaks, including outbreaks which spread according to spatial
adjacency, adjacency plus simulated travel patterns, random graphs
(Erdos-Renyi and preferential attachment), and realistic simulations of
an anthrax bio-attack.  Our results demonstrated that two of our
approaches, GS-PsCorr and ULS-GrCorr, consistently outperformed the
other three approaches in terms of spatial accuracy, timeliness of
detection, and accuracy of the learned graph structure.  Both GS-PsCorr
and ULS-GrCorr consistently achieved more timely and more accurate event
detection than the recently proposed MultiTree
algorithm~\citep{Gomez-Redriguez:2012:MultiTree}, even when MultiTree was
provided with labeled data not available to our algorithms.  We observed
a tradeoff between scalability and detection: GS-PsCorr had slightly
better detection performance than ULS-GrCorr, while ULS-GrCorr was able
to scale to larger graphs (500 nodes vs.~100 nodes).  None of
our approaches are designed to scale to massive graphs with millions of
nodes (e.g.~online social networks); they are most appropriate for
moderate-sized graphs where labeled data is not available and timely,
accurate event detection is paramount.

In general, our results demonstrate that the graph structures learned by
our framework are similar to the true underlying graph structure,
capturing nearly all of the true edges but also adding some additional
edges.  The resulting graph achieves similar spatial accuracy to the true
graph, as measured by the overlap coefficient between true and detected
clusters.  Interestingly, the learned graph often has \emph{better}
detection power than the true underlying graph, enabling more timely
detection of outbreaks or other emerging events.  This result can be
better understood when we realize that the learning procedure is designed
to capture not only the underlying graph structure, but the
characteristics of the events which spread over that graph. Unlike
previously proposed methods, our framework learns these characteristics
from unlabeled training examples, for which we assume that an event is
occurring but are not given the affected subset of nodes.  By finding
graphs where the highest connected subgraph score is consistently close to
the highest unconstrained subset score when an event is occurring, we
identify a graph structure which is optimized for event detection.
Our ongoing work focuses on extending the graph structure learning
framework in several directions, including learning graph structures
with directed rather than undirected edges, learning graphs with
weighted edges, and learning dynamic graphs where the edge structure can
change over time.

\section*{Acknowledgments}
This work was partially supported by NSF grants IIS-0916345,
IIS-0911032, and IIS-0953330. Preliminary work was presented at the 2011
International Society for Disease Surveillance Annual Conference, with a
1-page abstract published in the \emph{Emerging Health Threats Journal}.
This preliminary work did not include the theoretical
developments and results, the computational algorithmic advances, and the
large set of comparison methods and evaluations considered here.

\appendix

\section{Proofs of Lemma 1 and Lemma 2}
\label{appendix_proof}

We begin with some preliminaries which will be used in both proofs.
Following the notation in~\citet{neill-ltss}, we write the distributions
from the exponential family as $\log P(x\:|\:\mu) = T(x)\theta(\mu) -
\psi(\theta(\mu)) = T(x)\theta(\mu) - \mu\theta(\mu) + \phi(\mu)$, where
$T(x)$ is the sufficient statistic, $\theta(\mu)$ is a function mapping
the mean $\mu$ to the natural parameter $\theta$, $\psi$ is the
log-partition function, and $\phi$ is the convex conjugate of $\psi$. By
assumption (A2), $F(S)$ is an expectation-based scan statistic in the
\emph{separable exponential family}, defined by~\citet{neill-ltss} as
follows:

\newtheorem{defn}{Definition}
\begin{defn} The separable exponential family is a subfamily of the exponential family such that $\theta(q\mu_i) = z_i \theta_0(q) + v_i$, where the function $\theta_0$ depends only on $q$, while $z_i$ and $v_i$ can depend on $\mu_i$ and $\sigma_i$ but are independent of $q$.
\end{defn}

\noindent Such functions can be written in the form $F(S) = \max_{q>1} \sum_{s_i \in S} \lambda_i(q)$, where:
\[\lambda_i(q) = T(x_i)z_i(\theta_0(q) - \theta_0(1)) + \mu_i z_i\left(\theta_0(1) - q \theta_0(q) + \int_1^q \theta_0(x)\: dx\right).\]

\citet{speakman15} have shown that $\lambda_i(q)$ is a concave function with global maximum at $q = q_i^{mle}$ and zeros at $q=1$ and
$q=q_i^{max}$, where $q_i^{mle} = \frac{T(x_i)}{\mu_i}$ and $q_i^{max}$ is an increasing function of $q_i^{mle}$.  Considering the corresponding
excess risks $r_i^{mle} = q_i^{mle} - 1$ and $r_i^{max} = q_i^{max} - 1$, we know:

\begin{equation}
\label{r_mapper}
r_i^{mle} = r_i^{max} \left( \frac{\theta_0(r_i^{max}+1) - \bar\theta_0}{\theta_0(r_i^{max}+1) - \theta_0(1)} \right),
\end{equation}
where $\bar\theta_0 = \frac{1}{r_i^{max}} \int_1^{r_i^{max}+1} \theta_0(x) dx$ is the average value of $\theta_0$ between 1 and $r_i^{max}+1$.

From this equation, it is easy to see that $r_i^{mle} \le \frac{r_i^{max}}{2}$ when $\theta_0$ is concave, as is the case for the Poisson, Gaussian,
and exponential distributions, with $\theta_0(q) = \log(q)$, $q$, and $-\frac{1}{q}$ respectively.  For the Gaussian, $r_i^{mle} = \frac{r_i^{max}}{2}$
since $\theta_0$ is linear, while $r_i^{mle} < \frac{r_i^{max}}{2}$ for the Poisson and exponential.

Further, the assumption of an expectation-based scan statistic in the separable exponential family (A2) implies that the score function $F(S)$ satisfies
the linear-time subset scanning property~\citep{neill-ltss} with priority function $G(v_i) = \frac{T(x_i)}{\mu_i}$.  This means that the highest-scoring
unconstrained subset $S_j^\ast = \arg\max_S F(S)$ can be found by evaluating the score of only $|V|$ of the $2^{|V|}$ subsets of nodes, that is,
$S_j^\ast = \{v_{(1)},v_{(2)},\ldots,v_{(k)}\}$ for some $k$ between 1 and $|V|$, where $v_{(i)}$ represents the $i$th highest-priority node.

Given the set of all nodes $\{v_{(1)}, v_{(2)}, \ldots, v_{(|V|)}\}$ sorted by priority, we note that the assumption of a 1-strong signal implies that the true
affected subset $S_j^T = \{v_{(1)},v_{(2)},\ldots,v_{(t)}\}$, where $t$ is the cardinality of $S_j^T$.  Thus, for Lemma 1 we need only to show that $|S_j^\ast| \ge t$,
while for Lemma 2 we must show $|S_j^\ast| \le t$.  We can now prove:

\twolemma*
\begin{proof}

Let $\alpha_j = \frac{r_{\max}^{\text{aff},j}}{f(r_{\max}^{\text{aff},j})}$, where $f(r_i^{max}) = r_i^{mle}$ is the function defined in Equation (\ref{r_mapper})
above.  For distributions with concave $\theta_0(q)$, such as the Poisson, Gaussian, and exponential, we know that $f(r) \le \frac{r}{2}$, and thus $\alpha_j \ge
2$.  Now, the assumption of $\alpha_j$-homogeneity implies $\frac{r_{\max}^{\text{aff},j}}{r_{\min}^{\text{aff},j}} <
\frac{r_{\max}^{\text{aff},j}}{f(r_{\max}^{\text{aff},j})}$, $r_{\min}^{\text{aff},j} > f(r_{\max}^{\text{aff},j})$, and since $f(r)$ is an increasing and therefore
invertible function, $f^{-1}(r_{\min}^{\text{aff},j}) > r_{\max}^{\text{aff},j}$.

Now we note that $r_{\min}^{\text{aff},j}$ is the observed excess risk $\frac{T(x_i)}{\mu_i} - 1$ for the lowest-priority affected node $v_{(t)}$, where $t$ is the cardinality of $S_j^T$, while
$r_{\max}^{\text{aff},j}$ is the observed excess risk for the highest-priority affected node $v_{(1)}$.  Moreover, the contribution of node $v_{(t)}$ to the log-likelihood
ratio statistic, $\lambda_t(q)$, will be positive for all $q < 1 + f^{-1}(r_{\min}^{\text{aff},j})$, and we know that the maximum likelihood estimate of $q$ for any
subset of nodes $\{v_{(1)},v_{(2)},\ldots,v_{(k)}\}$ will be at most $q = 1 +  r_{\max}^{\text{aff},j} < 1 + f^{-1}(r_{\min}^{\text{aff},j})$.  Thus node $v_{(t)}$ will make a positive contribution to the
log-likelihood ratio and will be included in $S_j^\ast$, as will nodes $v_{(1)} \ldots v_{(t-1)}$.  Hence $|S_j^\ast| \ge t$, and $S_j^\ast \supseteq S_j^T$.
\end{proof}

\alphalemma*
\begin{proof}

Let $\beta_j = \frac{f^{-1}(r_{\max}^{\text{unaff},j})}{r_{\max}^{\text{unaff},j}}$, where $f^{-1}(r_i^{mle}) = r_i^{max}$ is the inverse of the function defined in Equation (\ref{r_mapper})
above.  For distributions with convex $\theta_0(q)$, such as the Gaussian, we know that $f^{-1}(r) \le 2r$, and thus $\beta_j \le 2$.
Now, the assumption that the signal is $\frac{\beta_j}{\eta_j}$-strong, where $\eta_j = \frac{\sum_{v_i \in S_j^T} \mu_i}{\sum_{v_i} \mu_i}$,
implies $\frac{r_{\min}^{\text{aff},j}}{r_{\max}^{\text{unaff},j}}  > \frac{f^{-1}(r_{\max}^{\text{unaff},j})}{\eta_j r_{\max}^{\text{unaff},j}}$
and thus $\left(\frac{\sum_{v_i \in S_j^T} \mu_i}{\sum_{v_i} \mu_i}\right) r_{\min}^{\text{aff},j} > f^{-1}(r_{\max}^{\text{unaff},j})$.

Now we note that $r_{\min}^{\text{aff},j}$ is the observed excess risk $g_{ij} = \frac{T(x_i)}{\mu_i} - 1$ for the lowest-priority affected node $v_{(t)}$,
and $r_{\max}^{\text{unaff},j}$ is the observed excess risk for the highest-priority unaffected node $v_{(t+1)}$, where $t$ is the cardinality of $S_j^T$.
Moreover, the contribution of node $v_{(t+1)}$ to the log-likelihood ratio statistic, $\lambda_{t+1}(q)$, will be negative for all $q > 1 + f^{-1}(r_{\max}^{\text{unaff},j})$.
Finally, we know that the maximum likelihood estimate of $q$ for any $\{v_{(1)},v_{(2)},\ldots,v_{(k)}\}$ will be at least $q = \frac{\sum_{v_i} T(x_i)}{\sum_{v_i} \mu_i} = 1 + r$, where $r =
\frac{\sum_{v_i} g_{ij}\mu_i}{\sum_{v_i} \mu_i} = \frac{\sum_{v_i \in S_j^T} g_{ij}\mu_i + \sum_{v_i \not\in S_j^T}g_{ij}\mu_i}{\sum_{v_i} \mu_i} >
\frac{\sum_{v_i \in S_j^T} r_{\min}^{\text{aff},j} \mu_i}{\sum_{v_i} \mu_i} > f^{-1}(r_{\max}^{\text{unaff},j})$, where the key step is to lower bound
each $g_{ij}$ by $r_{\min}^{\text{aff},j}$ for $v_i \in S_j^T$ and by 0 for $v_i \not\in S_j^T$ respectively.
Thus node $v_{(t+1)}$ will make a negative contribution to the
log-likelihood ratio and will be excluded from $S_j^\ast$, as will nodes $v_{(t+2)} \ldots v_{(|V|)}$.  Hence $|S_j^\ast| \le t$, and $S_j^\ast \subseteq S_j^T$.
\end{proof}

\bibliography{sigproc}

\end{document}